\documentclass[preprint,12pt]{elsarticle}

\usepackage{amssymb}
\usepackage{amsmath}
\usepackage{mathrsfs}
\usepackage{mathtools}
\usepackage{amsthm}
\usepackage{algorithm}
\usepackage{algorithmic}

\usepackage{booktabs} 
\usepackage{url} 
\usepackage{multirow}

\theoremstyle{plain}
\newtheorem{theorem}{Theorem}

\newtheorem{lemma}{Lemma}

\theoremstyle{definition}
\newtheorem{definition}{Definition}

\theoremstyle{remark}

\usepackage{xcolor}

\begin{document}

\begin{frontmatter}



\title{Spectral Clustering for Discrete Distributions}


\author[aff1]{Zixiao Wang}
\author[aff1]{Dong Qiao}
\author[aff1,aff2]{Jicong Fan\corref{cor1}}
\ead{fanjicong@cuhk.edu.cn}

\affiliation[aff1]{organization={School of Data Science, The Chinese University of Hong Kong, Shenzhen},
                    addressline={2001 Longxiang Boulevard, Longgang District}, 
                    city={Shenzhen},
                    postcode={518172}, 
                    state={Guangdong},
                    country={China}}
\affiliation[aff2]{organization={Shenzhen Research Institute of Big Data},
                    addressline={2001 Longxiang Boulevard, Longgang District}, 
                    city={Shenzhen},
                    postcode={518172}, 
                    state={Guangdong},
                    country={China}}
\cortext[cor1]{Corresponding author.}

\begin{abstract}
The discrete distribution is often used to describe complex instances in machine learning, such as images, sequences, and documents. Traditionally, clustering of discrete distributions (D2C) has been approached using Wasserstein barycenter methods. These methods operate under the assumption that clusters can be well-represented by barycenters, which is seldom true in many real-world applications. Additionally, these methods are not scalable for large datasets due to the high computational cost of calculating Wasserstein barycenters. In this work, we explore the feasibility of using spectral clustering combined with distribution affinity measures (e.g., maximum mean discrepancy and Wasserstein distance) to cluster discrete distributions. We demonstrate that these methods can be more accurate and efficient than barycenter methods. To further enhance scalability, we propose using linear optimal transport to construct affinity matrices efficiently for large datasets. We provide theoretical guarantees for the success of our methods in clustering distributions. Experiments on both synthetic and real data show that our methods outperform existing baselines.
\end{abstract}




\end{frontmatter}


\section{Introduction and Related Works}\label{sec:introduction}
Clustering is a fundamental task in machine learning, with various methods developed over the decades \citep{shi2000normalized,Jain_Murty_Flynn_2002,ng2002spectral,sc2004,Kriegel2009,SSC_PAMIN_2013,zhang2015smart, campello2015hierachical, 10.1145/3132088,zhang2018binary,HUANG2021107996,kang2021structured,fankdd21, cai2022efficient,NEURIPS2022_fan, sun2023laplacian,ZHAO2023109836}. raditional methods like k-means \cite{macqueen1967some,bezdek1984fcm}, spectral clustering \cite{shi2000normalized,ng2002spectral}, subspace clustering \cite{bradley2000k,parsons2004subspace}, and deep clustering \cite{xie2016unsupervised,cai2022efficient} focus on clustering individual samples represented as vectors. However, many data types are not naturally vectors and are better described as discrete distributions. For instance, the bag-of-words model for documents is a discrete distribution.

\textbf{However, many data types are intrinsically not vectors}. Instead, they can be naturally described as discrete distributions. For example, the widely used a-bag-of-word data model represents each document as a discrete distribution. To cluster discrete distributions, one may consider representing each batch of data as a vector (e.g., the mean value of the batch), which will lose important information about the batch and lead to unsatisfactory clustering results. Alternatively, one may cluster all the data points without considering the batch condition and use post-processing to obtain the partition over batches, which does not explicitly exploit the distribution of each batch in the learning stage and leads to unsatisfactory performance. Thus, it is crucial to develop efficient and effective methods for clustering discrete distributions rather than relying on pre-processing or post-processing heuristics.

Several methods based on the \textit{D2-Clustering} framework \citep{Li_Wang_2007} have been developed for clustering discrete distributions. D2-Clustering operates similarly to traditional K-means but with key differences: it clusters discrete distributions instead of vectors, uses the Wasserstein distance instead of Euclidean distance, and calculates Wasserstein barycenters instead of K-means centroids. The Wasserstein distance, derived from optimal transport theory \citep{villani2009optimal}, captures meaningful geometric features between probability measures and effectively compares discrete distributions. The Wasserstein barycenter \citep{wassersteinbarycenter2011} of multiple discrete distributions is a discrete distribution that minimizes the sum of Wasserstein distances to the given distributions. Despite the competitive practical performance of D2-Clustering, it has several limitations:
\begin{itemize}
    \item \textbf{Hardness of Wasserstein distance calculation.} The Wasserstein distance has no closed-form solution. Given two discrete distributions, each with $m$ supports, the Wasserstein distance between them can only be solved with a worst-case time complexity $O(m^3 \log m)$ \citep{orlin1988faster}.
    \item \textbf{Hardness of barycenter calculation.} Solving the true discrete Wasserstein barycenter quickly is intractable even for a small number of distributions that contains a small number of support \citep{Anderes_Borgwardt_Miller_2015}. A large number of representative initial supports are required to achieve a good barycenter approximation.
    \item \textbf{Unrealistic assumption.} D2-Clustering assumes the distributions concentrate at some centroids, which may not hold in real applications. Please refer to Figure~\ref{fig:toy_d2c} for an intuitive explanation.
    \item \textbf{Lacking theoretical guarantees} The study on theoretical guarantees for the success of D2-Clustering is very limited. It is unclear under what conditions we can cluster the distributions correctly.
\end{itemize}

We address the aforementioned issues in this work. Our contributions are two-fold:
\begin{itemize}
    \item  We propose to use spectral clustering and distribution affinity measures (e.g., Maximum Mean Discrepancy and Wasserstein distance) to cluster discrete distribution, which yields new methods (four new algorithms) that are extremely easy to implement and do not need careful initialization. 
    \item We provide theoretical guarantees for the consistency and correctness of clustering for our proposed methods. The theoretical results provide strong support for practical applications.
\end{itemize}
We evaluate our methods' accuracy, robustness, and scalability compared to baselines on both synthetic and real data. Empirical results show that our methods outperform the baselines remarkably.

\textbf{Notations} ~Let $\Omega$ be a compact subset of $\mathbb{R}^d$, and let $P(\Omega)$ be the set of Borel probability measures on $\Omega$.The diameter of $\Omega$, denoted as $|\Omega|$, is defined as $\sup\{|\mathbf{x} - \mathbf{x}'| : \mathbf{x}, \mathbf{x}' \in \Omega\}$. $\mathcal{C}(\Omega)$ represents the set of all continuous functions on $\Omega$. For any $\mathbf{x} \in \Omega$, $\delta_\mathbf{x}$ denotes the Dirac unit mass at $\mathbf{x}$. Let $\mathbf{X} = (\mathbf{x}_1, \dots, \mathbf{x}_n)$ and $\mathbf{Y} = (\mathbf{y}_1, \dots, \mathbf{y}_m)$ be two sets of points in $\Omega$, which are finite samples from some underlying distributions. The points $\mathbf{x}_i$ and $\mathbf{y}_i$ are referred to as the support points of these distributions. For simplicity, assume $m = n$. The $n$-dimensional probability simplex is defined as $\Sigma_n := \{\mathbf{x} \in \mathbb{R}^n_+ \mid \mathbf{x}^\top\mathbf{1}_n = 1\}$. The cost function $c(\cdot, \cdot)$ belongs to $\mathcal{C}^{\infty}$ and is $L$-Lipschitz continuous.

With slight abuse of notation, let $\mathbf{X}_i = [\mathbf{x}_1, \mathbf{x}_2, \ldots, \mathbf{x}_{m_i}]^\top \in \mathbb{R}^{m_i \times d}$ denote samples independently drawn from some underlying distribution $\mu_i$ for $i = 1, 2, \ldots, N$. We use $\mathbf{D}$, $\mathbf{A}$, and $\mathbf{L}$ to denote the distance, adjacency, and Laplacian matrices, respectively. The norm $\|\cdot\|_\infty$ represents the maximum element of a matrix.

\section{Preliminary Knowledge}\label{sec_pre}

Measuring the divergence or distance between distributions is a fundamental problem in statistics, information theory, and machine learning \citep{kolouri2017optimal}. Common measures include Kullback-Leibler (KL) divergence, $\chi^2$-divergence, Maximum Mean Discrepancy (MMD) \citep{gretton2006kernel,gretton2012kernel}, Wasserstein distance \citep{panaretos2019statistical}, and Sinkhorn divergence \citep{cuturi2013Sinkhorn}. We will briefly introduce MMD, Wasserstein distance, and Sinkhorn divergence as they are widely used for comparing discrete distributions and will be integral to our proposed clustering methods.

\subsection{Maximum Mean Discrepancy}\label{sec_MMD}
Consider two probability distributions $\alpha$ and $\beta$, the maximum mean discrepancy between them is defined by
\begin{equation}
    \mathrm{MMD}^2(\alpha, \beta) := \mathbb{E}_{\alpha \otimes \alpha}[k(X, X')] + \mathbb{E}_{\beta \otimes \beta}[k(Y, Y')]
    - 2\mathbb{E}_{\alpha \otimes \beta}[k(X, Y)].
\end{equation}
where $k(\cdot, \cdot)$ is a user defined kernel (e.g. gaussian kernel). An empirical estimation with finite samples $\mathbf{X}, \mathbf{Y}$ from distributions $\alpha, \beta$ is
\begin{equation}
\begin{aligned}
\widehat{\mathrm{MMD}}^2(\mathbf{X}, \mathbf{Y}) :=& 
 \frac{1}{m(m-1)} \sum_{i=1}^m \sum_{j \neq i}^m k\left(\mathbf{x}_i, \mathbf{x}_j\right)
+\frac{1}{m(m-1)} \sum_{i=1}^m \sum_{j \neq i}^m k\left(\mathbf{y}_i, \mathbf{y}_j\right) \\
&-\frac{2}{m^2} \sum_{i=1}^m \sum_{j=1}^m k\left(\mathbf{x}_i, \mathbf{y}_j\right).
\end{aligned}
\end{equation}


\subsection{Wasserstein Distance}\label{sec_WD}
Wasserstein distance has been used extensively in recent years as a powerful tool to compare distributions. It originates from the famous optimal transport (OT) problem. Consider two probability measures $\alpha, \beta \in P(\Omega)$, the Kantorovich formulation \citep{kantorovich1960mathematical} of OT between $\alpha$ and $\beta$ is
\begin{equation}
\min _{\pi \in \Pi(\alpha, \beta)} \int_{\Omega^2}c(x, y) \mathrm{~d} \pi(x, y),
\end{equation}
where $\Pi(\alpha, \beta)$ is the set of all probability measures on $\Omega^2$ that has marginals $\alpha$ and $\beta$. $c(x, y)$ is the cost function representing the cost to move a unit of mass from $x$ to $y$. The $k$-Wasserstein distance is defined as
\begin{equation}
    W_k(\alpha, \beta):=\left( \min _{\pi \in \Pi(\alpha, \beta)} \int_{\Omega^2} \|x - y\|^k \mathrm{~d} \pi(x, y)\right)^{1 / k}.
\end{equation}
In practice, we calculate its discrete approximation using the distributions' finite samples $\mathbf{X}$ and $\mathbf{Y}$. For $\hat{\alpha} = \sum_{i=1}^m a_i \delta_{x_i}$ and $\hat{\beta} = \sum_{i=1}^m b_i \delta_{y_i}$ where $\mathbf{a} \in \Sigma_m$ and $\mathbf{b} \in \Sigma_m$, the $k$-Wasserstein distance is defined by
\begin{equation}
\hat{W}_k(\hat{\alpha}, \hat{\beta}):= \left(\min _{\mathbf{P} \in U(a, b)}\langle \mathbf{P}, \mathbf{C}\rangle\right)^{1/k},
\end{equation}
where $U(\mathbf{a}, \mathbf{b})=\left\{\mathbf{P} \in \mathbb{R}_{+}^{m \times m} \mid \mathbf{P} \mathbf{1}_m=\mathbf{a}, \mathbf{P}^T \mathbf{1}_m=\mathbf{b}\right\}$ and $\mathbf{C}$ is denotes cost matrix whose $(i, j)$-th element $c_{ij} = \|x_i - y_j\|^k$. All Wasserstein distances used in the following content are $2$-Wasserstein distance, so we omit the subscript $k$ to simplify notations.


\subsection{Sinkhorn Divergences}\label{sec_Sinkhorn}
To mitigate the high computational cost of optimal transport (OT), \citet{cuturi2013Sinkhorn} proposed an approximation formulation of the classic OT problem with an entropic regularization term. This approach results in a divergence that can be efficiently computed with a cost of $O(m^2)$ per iteration.

The Sinkhorn divergence is defined as
\begin{equation}\label{def:Sinkhorn}
    W_\varepsilon(\alpha, \beta) := 
    \min _{\pi \in \Pi(\alpha, \beta)} \int_{\Omega^2}c(x, y) \mathrm{~d} \pi(x, y) 
    + \varepsilon H(\pi | \alpha \otimes \beta)
\end{equation}
where $H(\pi | \alpha \otimes \beta):= -\int_{\Omega^2} \log\left(\frac{\mathrm{~d}\pi(x, y)}{\mathrm{~d}\alpha(x)\mathrm{~d}\beta(y)}\right) \mathrm{~d} \pi(x, y)$ \citep{genevay2016stochastic}.

Its discrete estimation version is similarly defined as
\begin{equation}
\hat{W}_\varepsilon(\hat{\alpha}, \hat{\beta}) := \min_{\mathbf{P} \in U(a, b)} \langle \mathbf{P}, \mathbf{C}\rangle + \varepsilon H(\mathbf{P} | \hat{\alpha} \otimes \hat{\beta}),
\end{equation}
where $H(\mathbf{P} | \hat{\alpha} \otimes \hat{\beta}) := -\sum_{i=1}^m\sum_{j=1}^m p_{ij} \log \frac{p_{ij}}{a_i b_j}$.

\subsection{Wasserstein Barycenter and D2-Clustering}\label{sec_WB}
Discrete distribution (D2) clustering was first proposed by \citet{Li_Wang_2007} for automatic picture annotation. D2 clustering adopts the same spirit as K-means clustering. Instead of vectors, D2 clustering aims at finding clusters for distributions. It computes some centroids to summarize each cluster of data. The centroids are known as Wasserstein barycenter (WB) which themselves are also discrete distributions. For a set of discrete distributions $\{\nu_{i}, i = 1, \dots, N\}$, the WB is the solution of the following problem
\begin{equation}
    \min_{\mu} \sum_{i=1}^N W^2(\mu, \nu_i).
\end{equation}
The goal of D2 clustering is to find a set of such centroids $\{\mu_{k}, k = 1, \dots, K\}$ such that the total within-cluster variation, which is measured by Wasserstein distance, is minimized:
\begin{equation}
    \min_{\{\mu_k\}} \sum_{k=1}^K \sum_{i \in C_k} W(\mu_k, \nu_i)
\end{equation}
where $C_k$ is the set of indices of distributions assigned to cluster $k$ and $\mu_k$ is the WB of those distributions.

\section{Discrete Distribution Spectral Clustering}

\subsection{Motivation and Problem Formulation}
As mentioned in Section \ref{sec_WB}, existing methods use Wasserstein barycenters to cluster discrete distributions similarly to K-means clustering. These methods assume that members in each cluster can be well-determined by the Wasserstein distances to the barycenters, making them distance-based methods. However, in real applications, distributions may not concentrate around central points but instead form irregular shapes or lie on multiple manifolds. In such cases, connectivity-based methods like spectral clustering are more effective. Figure \ref{fig:toy_d2c} shows an intuitive comparison between distance-based and connectivity-based clustering of distributions, with ellipses and rectangles representing two different types of distributions. 
\begin{figure}[t]
\centering
\includegraphics[width=1\linewidth]{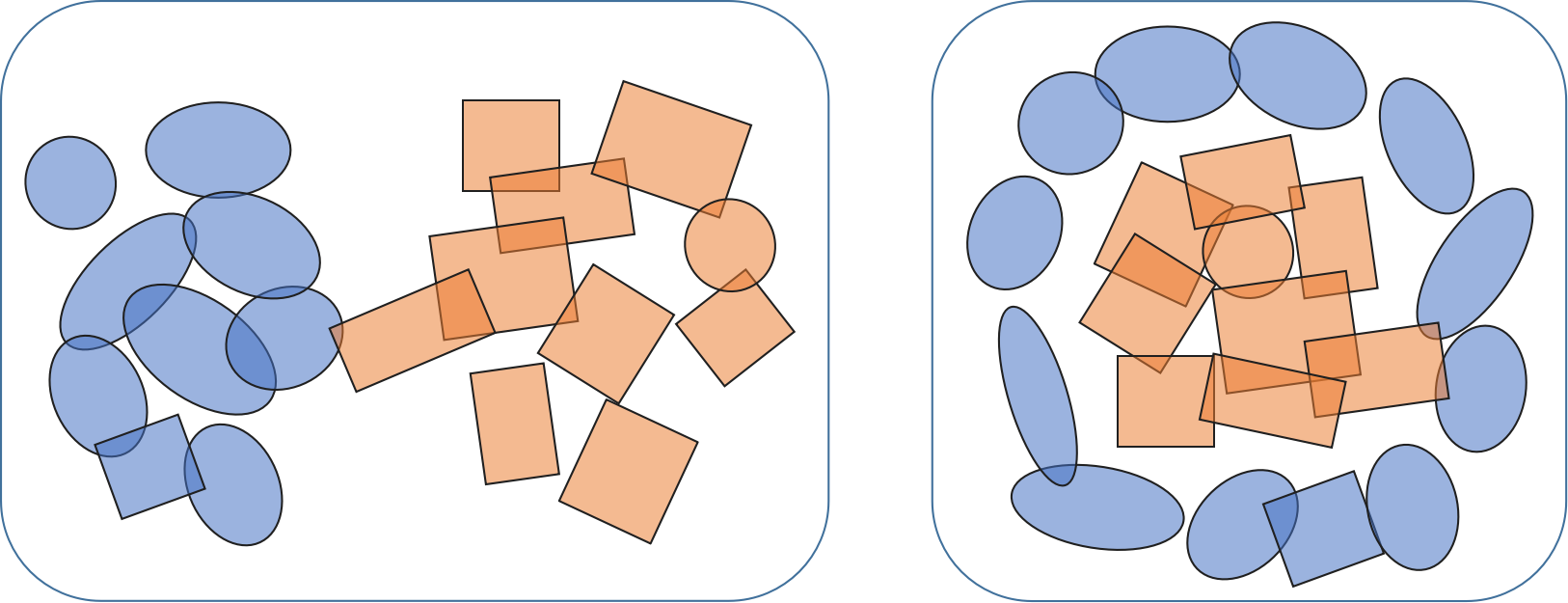}
\caption{An intuitive comparison between distance-based (left plot) clustering and connectivity-based (right plot) clustering of distributions. Ellipses and rectangles denote distributions. The two clusters are marked in different colors.}
\label{fig:toy_d2c}
\end{figure}

We first make the following formal definitions of what we mean by connectivity-based distribution clustering:
\begin{definition}[Connectivity-based distribution clustering]\label{def_cbdc}
    Given $N$ distributions $\mu_1,\mu_2,\ldots,\mu_N$ in that can be organized into $K$ groups $C_1,C_2,\ldots,C_K$. Let $\text{dist}(\cdot,\cdot)$ be a distance metric between two distributions and use it to construct a $\tau$ nearest neighbor ($\tau$-NN) graph $G=(V,E)$ over the $N$ distributions, on which each node corresponds a distribution. We assume that the following properties hold for the graph: 1) $G$ has $K$ connected components;
    2) $(i,j)\notin E$ if $\mu_i$ and $\mu_j$ are not in the same group. 
\end{definition}
The two properties in the definition are sufficient to guarantee that partitioning the graph into $K$ groups yields correct clustering for the $N$ distributions. Note that the definition is actually a special case of connectivity-based clustering because there are many other approaches or principles for constructing a graph. In addition, it is possible that a correct clustering can be obtained even when some $(i,j)\in E$ but $\mu_i$ and $\mu_j$ are not in the same group. In this work, we aim to solve the following problem.
\begin{definition}[Connectivity-based discrete distribution clustering]\label{def_cbdc_1}
  Based on Definition \ref{def_cbdc}, suppose $\mathbf{X}_i=[\mathbf{x}_1,\mathbf{x}_2,\ldots,\mathbf{x}_{m_i}]^\top \in \mathbb{R}^{m_i\times d}$ are independently drawn from $\mu_i$, $i=1,2,\ldots,N$. The goal is to partition $\mathbf{X}_1,\mathbf{X}_2,\ldots,\mathbf{X}_N$ into $K$ groups corresponding to $C_1,C_2,\ldots C_K$ respectively.
\end{definition}

\subsection{Algorithm}
To solve the problem defined in Definition \ref{def_cbdc_1}, we proposed to use MMD (Section \ref{sec_MMD}), Wasserstein distance (Section \ref{sec_WD}), and Sinkhorn divergence (Section \ref{sec_Sinkhorn}) to construct a distance matrix $\mathbf{D}$ between the $N$ distributions: 
\begin{equation}
 D_{ij}=\text{dist}(\mathbf{X}_i,\mathbf{X}_j),\quad (i,j)\in[N]\times[N].   
\end{equation}
 The distance matrix $\mathbf{D}$ is then converted to an adjacency matrix $\mathbf{A}$ 
\begin{equation}\label{gaussian_kernel}
    A_{ij}=\exp\left(-\gamma D_{ij}^2\right), \quad (i,j)\in[N]\times[N],
\end{equation}
where $\gamma>0$ is a hyper-parameter. This operation is similar to using a Gaussian kernel to compute the similarity between two vectors in Euclidean space. The matrix $\mathbf{A}$ is generally a dense matrix although some elements could be close to zero when $\gamma$ is relatively large. We hope that $A_{ij}=0$ if $\mathbf{X}_i$ and $\mathbf{X}_j$ are not in the same cluster. Therefore, we propose to sparsify $\mathbf{A}$, i.e.,
\begin{equation}
    \mathbf{A}\leftarrow \text{sparse}_{\tau}(\mathbf{A})
\end{equation}
where we keep only the largest $\tau$ elements of each column of $\mathbf{A}$, replace other elements with zeros, and symmetrize the matrix as $\mathbf{A} \leftarrow (\mathbf{A}+\mathbf{A}^\top)/2$. Finally, we perform the normalized cut \citep{shi2000normalized,ng2002spectral} on $\mathbf{A}$ to obtain $K$ cluster. The steps are detailed in Algorithm \ref{alg.sc}. We call the algorithm DDSC$_{\text{MMD}}$, DDSC$_{\text{Sinkhorn}}$, or DDSC$_{\text{W}}$ when MMD, Sinkhorn, or 2-Wasserstein distance is used as $\text{dist}(\cdot,\cdot)$ respectively.

\begin{algorithm}[ht] 
\caption{Discrete Distributions Spectral Clustering (DDSC)} \label{alg.sc}
\begin{algorithmic}[1]
\REQUIRE $\{\mathbf{X}_1,\mathbf{X}_2,\ldots,\mathbf{X}_N\}$, $\text{dist}(\cdot,\cdot)$ (e.g., MMD and Sinkhorn), $K$, $\tau$
\STATE Initialization: $\mathbf{D}=\mathbf{0}_{N \times N}$
\FOR{$i=1,2,\ldots,N$}
	\FOR{$j=i+1,i+2,\ldots,N$}
          \STATE $D_{ij}=d(\mathbf{X}_i,\mathbf{X}_j)$
	\ENDFOR
\ENDFOR
\STATE $\mathbf{D}=\mathbf{D}+\mathbf{D}^\top$
\STATE $\mathbf{A}=\left[\exp(-\gamma D_{ij}^2)\right]_{N\times N}-\mathbf{I}_N$
\STATE $\mathbf{A}\leftarrow \text{sparse}_{\tau}(\mathbf{A})$ 
\STATE $\mathbf{S}=\text{diag}(\sum_{i}A_{i1},\sum_{i}A_{i2},\ldots,\sum_{i}A_{iN})$
\STATE $\mathbf{L}=\mathbf{I}_N-\mathbf{S}^{-1/2}\mathbf{A}\mathbf{S}^{-1/2}$
\STATE Eigenvalue decomposition: $\mathbf{L}=\mathbf{V}\mathbf{\Lambda}\mathbf{V}^\top$, where $\lambda_1\leq\lambda_2\cdots\leq \lambda_N$
\STATE $\mathbf{V}_K=[\mathbf{v}_1,\mathbf{v}_2,\ldots,\mathbf{v}_K]$
\STATE Normalize the rows of $\mathbf{V}_K$ to have unit $\ell_2$ norm.
\STATE Perform $K$-means on $\mathbf{V}_K$.
\ENSURE $K$ clusters: $C_1,\ldots,C_K$.
\end{algorithmic}
\end{algorithm}
\subsection{Linear Optimal Transport}
Algorithm~\ref{alg.sc} proposed above requires the pairwise calculation of optimal transport distance. For a dataset with $N$ distributions, ${N(N-1)}/{2}$ calculations are required. It is computationally expensive, especially for large datasets. In what follows, we apply the linear optimal transportation (LOT) framework \citep{wang2013linear} to reduce the calculation of optimal transport distances from quadratic complexity to linear complexity. \citet{wang2013linear} described a framework for constructing embeddings of probability measures such that the Euclidean distance between the embeddings approximates 2-Wasserstein distance. We leverage the prior works (\citet{wang2013linear, kolouri2016continuous}) with some modifications and introduce how we use the framework to make our clustering algorithm more efficient. Here we mainly focus on the linear Wasserstein embedding for continuous measures, but all derivations hold for discrete measures as well.

Let $\mu_0$ be a reference probability measure (or template) with density $p_0$, s.t. $\mathrm{~d}\mu_0(x) = p_0(x) \mathrm{~d}x$. Let $f_i$ be the Monge optimal transport map between $\mu_0$ and $\mu_i$, we can define a mapping $\phi(\mu_i) := (f_i -id)\sqrt{p_0}$, where $id(x)=x$ is the identity function. The mapping $\phi(\cdot)$ has the following properties:
\begin{itemize}
    \item $\phi(\mu_0)=0$;
    \item $\|\phi(\mu_i) - \phi(\mu_0)\| = \|\phi(\mu_i)\| = W_2(\mu_i, \mu_0)$, i.e., the mapping preserves distance to template $\mu_0$;
    \item $\|\phi(\mu_i) - \phi(\mu_j)\| \approx W_2(\mu_i, \mu_j)$, i.e., the Euclidean distance between $\phi(\mu_i), \phi(\mu_j)$ is an approximation of $W_2(\mu_i, \mu_j)$.
\end{itemize}

In practice, for discrete distributions, we use Monge coupling $\mathbf{f}_i$ instead of Monge map $f_i$. The Monge coupling could be approximated from the Kantorovich plan. Specifically, for discrete distributions $\{\mathbf{X}_1,\mathbf{X}_2,\ldots,\mathbf{X}_N\}$ where $\mathbf{X}_i\in\mathbb{R}^{m_i\times d}$, reference distribution $\mathbf{X}_0 \in \mathbb{R}^{m_0 \times d}$ where $m_0 = \frac{1}{N}\sum_{i=1}^Nm_i$, with Kantorovich transport plans $[\mathbf{g}_1, \mathbf{g}_2, \ldots, \mathbf{g}_N]$ where $\mathbf{g}_i \in \mathbb{R}^{m_i \times d}$ is the transport plan between $\mathbf{X}_0$ and $\mathbf{X}_i$, the Monge coupling is approximated as
\begin{equation}
    \mathbf{f}_i = m_0(\mathbf{g}_i\mathbf{X}_i) \in \mathbb{R}^{m \times d}.
\end{equation}
The embedding mapping $\phi(\mathbf{X}_i) = (\mathbf{f}_i - \mathbf{X}_0) / \sqrt{m_0}$. We now formally present Algorithm~\ref{alg.lotsc} for discrete distribution clustering which leverages the LOT framework.
\begin{algorithm}[ht]
\caption{DDSC with LOT} \label{alg.lotsc}
\begin{algorithmic}[1]
\REQUIRE $\{\mathbf{X}_1,\ldots,\mathbf{X}_N\}$,  $\text{dist}(\cdot,\cdot)$ (e.g., MMD or Sinkhorn), $\mathbf{X}_i\in\mathbb{R}^{m_i\times d}$, $K$, $\tau$
\STATE Initialization: $\mathbf{D}=\mathbf{0}_{N \times N}$
\STATE Initialization: $m_0 = \frac{1}{N}\sum_{i=1}^Nm_i$, $\mathbf{X}_0 \in \mathbb{R}^{m_0 \times d}$
\STATE Initialization: $\mathbf{g} = [\mathbf{g}_1, \mathbf{g}_2, \ldots, \mathbf{g}_N], \mathbf{g}_i \in \mathbb{R}^{m_0 \times m_i}$ (list of Kantorovich plans)
\FOR {$k=1,2,\ldots,N$}
    \STATE $\mathbf{g}_k \leftarrow \text{Transport plan for } W_2(\mathbf{X}_0, \mathbf{X}_i)$
\ENDFOR
\FOR{$i=1,2,\ldots,N$}
    \STATE $\mathbf{Z}_i = \phi(\mathbf{X}_i)$
    \FOR{$j=i+1,i+2,\ldots,N$}
          \STATE $\mathbf{Z}_j = \phi(\mathbf{X}_j)$
          \STATE $D_{ij} = \|\mathbf{Z}_i - \mathbf{Z}_j\|_F$
    \ENDFOR
\ENDFOR
\STATE Follow lines 7-15 of Algorithm \ref{alg.sc}.
\ENSURE $K$ clusters: $C_1,\ldots,C_K$.
\end{algorithmic}
\end{algorithm}

We can see that for $N$ distributions $\{\mu_i\}_{i=1}^N$, we only need to solve $N$ optimal transport problems to embed the original distributions, and ${N(N-1)}/{2}$ Frobenius norms to get the distance matrix. For the initialization of $\mathbf{X}_0$, we use the normal distribution for simplicity. Indeed, it is shown (numerically) by \citep{kolouri2020wasserstein} that the choice of reference distribution is insignificant. Algorithm \ref{alg.lotsc} is called DDSC$_{\text{LOT}}$.

\section{Theoretical Guarantees}
As mentioned in Section~\ref{sec_pre}, we calculate discrete approximations of distances using finite samples in practice, leading to errors in the distance matrix $\mathbf{D}$ due to sampling. Therefore, it is meaningful and necessary to provide a theoretical guarantee for \textit{consistent and correct} clustering, which, to our knowledge, has not been addressed in previous discrete distribution clustering works.

Our discussion will primarily focus on Sinkhorn divergence, the most general case. Similar results for Wasserstein distance ($\varepsilon \to 0$) and MMD ($\varepsilon \to \infty$) can be derived trivially. For simplicity, we assume the distributions have the same number of support points, i.e., $m_i = m$ for all $i = 1, \dots, N$.

Before our analysis, we first introduce an important Lemma. As mentioned in \citep{genevay2016stochastic}, the Sinkhorn divergence defined in \eqref{def:Sinkhorn} has a dual formulation:
\begin{equation}
\begin{aligned}
        W_\varepsilon(\alpha, \beta) & = \max_{u,v\in\mathcal{C}(\Omega^2)} \int_\Omega u(x)\mathrm{~d}\alpha(x) + \int_\Omega v(y)\mathrm{~d}\beta(x) \\
        & \quad\quad- \varepsilon\int_{\Omega^2}\exp(\frac{u(x)+v(y)-c(x, y)}{\varepsilon})\mathrm{~d}\alpha(x)\mathrm{~d}\beta(y) + \varepsilon\\ 
        & = \mathbb{E}_{\alpha \otimes \beta}f^{XY}_\varepsilon(u, v) + \varepsilon
\end{aligned}
\end{equation}
where $X, Y$ are independent random variables distributed as $\alpha$ and $\beta$, respectively. $\mathcal{C}(\cdot)$ denotes the space of continuous function and $(u, v) \in \mathcal{C}(\Omega^2)$ are known as Kantorovich dual potentials. Here, $f^{XY}_\varepsilon(u, v) = u(x)+v(y)-\varepsilon \exp(\frac{u(x)+v(y)-c(x, y)}{\varepsilon})$ is a function of $u, v, \varepsilon$. 

\citet{genevay2019sample} showed that all the dual potentials $(u, v)$ are bounded in the Sobolev space $\mathbf{H}^s(\mathbb{R}^d)$ by some $\eta$ associated with $\varepsilon$. $f_\varepsilon$ is $B$-Lipschitz in $(u, v)$ on $\left\{(u, v) | u\oplus v \leq \kappa = 2L|\Omega| + \|c\|_\infty\right\}$. Let $\psi = \max_{x\in\Omega}k(x, x)$ where $k$ is the kernel associated with $\mathbf{H}^s(\mathbb{R}^d)$, we have the following  b\textbf{}lemma from \citep{genevay2019sample}
\begin{lemma}[Error bound of sampling-based Sinkhorn divergence]
\label{lemma:sinkhorn_err}
    Assume $D_{ij}$ is the $ij$-th entry of $\mathbf{D}$, then with probability at least $1 - \theta$
    \begin{equation}
        |D_{ij} - \hat{D}_{ij}| \leq \rho + E\sqrt{\log{\frac{1}{\theta}}},
    \end{equation}
    where $\kappa = 2L|\Omega| + \|c\|_\infty$, $\rho = 6B\frac{\eta\psi}{\sqrt{m}}$, and $E = \sqrt{\frac{2}{m}}\left(\kappa + \varepsilon\exp\left(\frac{\kappa}{\varepsilon}\right)\right)$.
\end{lemma}
The lemma is the core ingredient of the clustering consistency analysis in Section~\ref{consistency}. To understand the Lemma, we can treat $\rho$ and $E$ as some constants determined by the property of cost function $c(\cdot)$, data space $\Omega$, the number of support points $m$, and the regularization parameter $\varepsilon$. It is obvious that we get a more accurate Sinkhorn divergence estimation with more support points.

\subsection{Consistency Analysis}\label{consistency}
As stated before, we use only finite samples from the underlying distributions to calculate distances (or divergences), which introduces a gap between the estimated and actual values. Consequently, the estimated matrix $\mathbf{\hat{D}}$ is used instead of the ground-truth distance matrix $\mathbf{D}$ for spectral clustering. This may lead to an inconsistency between the results of our algorithm and the ideal clustering results. Therefore, it is necessary to analyze and discuss the clustering consistency of our methods, specifically whether we can achieve the same clustering results using sample estimations instead of ground-truth distances.

Since matrix $\mathbf{\hat{D}}$ is the estimation of ground-truth $\mathbf{D}$, we can only calculate the estimated adjacency (similarity) matrix $\mathbf{\hat{A}} = \exp(-\gamma \mathbf{\hat{D}} \odot \mathbf{\hat{D}})- \mathbf{I_N}$ and the estimated Laplacian matrix $\mathbf{\hat{L}}$, leading to the estimated eigenvector matrix $\mathbf{\hat{V}}$. However in spectral clustering, the eigenvector matrix $\mathbf{V}$ is critical for clustering results because the first $K$ eigenvectors are used to perform $K$-means. We argue that, a smaller difference between $\mathbf{V}$ and estimated eigenvector matrix $\hat{\mathbf{V}}$ will generate more consistent clustering results. Therefore, we need to analyze how sampling affects $\mathbf{V}$.

Our analysis is based on matrix perturbation theory. In perturbation theory, distances between subspaces are usually measured using so called ``principal angles". Let $\mathcal{V}_1, \mathcal{V}_2$ be two subspaces of $\mathbb{R}^d$, and $\mathbf{V}_1, \mathbf{V}_2$ be two matrices such that their columns form orthonormal basis for $\mathcal{V}_1$ and $\mathcal{V}_2$, respectively. Then the cosines $\cos{\Theta_i}$ of the principal angels $\Theta_i$ are defined as the singular values of $\mathbf{V}_1^T\mathbf{V}_2$. The matrix $\sin{\Theta(\mathcal{V}_1, \mathcal{V}_2)}$ is the diagonal matrix with the sine values of the principal angles on the diagonal. The distance of the two subspaces is defined by
\begin{equation}
    d(\mathbf{V}_1, \mathbf{V}_2) = \| \sin{\Theta(\mathcal{V}_1, \mathcal{V}_2)} \|_F
\end{equation}
Based on the distance of subspaces, we give a formal definition of clustering consistency:
\begin{definition}[Consistency of Clustering]
\label{clustering_consistency}
    Let $\mathbf{L}$ be Laplacian matrix constructed by true distance matrix $\mathbf{D}$ and $\hat{\mathbf{L}}$ be Laplacian matrix constructed by estimated distance matrix $\hat{\mathbf{D}}$. Let $\mathbf{V} = [\mathbf{v}_1, \mathbf{v}_2, \dots, \mathbf{v}_K]$ and $\hat{\mathbf{V}} = [\hat{\mathbf{v}}_1, \hat{\mathbf{v}}_2, \dots, \hat{\mathbf{v}}_K] \in \mathbb{R}^{N \times K}$ be two matrices whose columns are eigenvectors corresponding to the $K$-smallest eigenvalues of $\mathbf{L}$ and $\hat{\mathbf{L}}$, respectively. Then the clustering results generated by $\mathbf{D}$ and $\hat{\mathbf{D}}$ are $\xi$-consistent if 
    \begin{equation}
        d(\mathbf{V}, \hat{\mathbf{V}}) \leq \xi.
    \end{equation}
\end{definition}

\begin{theorem}[Consistency of DDSC$_\text{Sinkhorn}$]
\label{thm:consistency}
Suppose there exists an interval $S_1 \subset \mathbb{R}$ such that it contains only the $K$-smallest eigenvalues of Laplacian matrices $\mathbf{L}$ and $\hat{\mathbf{L}}$. Let $\alpha$ be the smallest non-zero element in the similarity matrix $\mathbf{A} \in \mathbb{R}^{N \times N}$, and $\delta = |\lambda_{K+1} - \lambda_K|$ be the eigen gap of the Laplacian matrix $\mathbf{L}$. Then with probability at least $1 - \theta$, performing spectral clustering with ground truth distance matrix $\mathbf{D}$ and estimated distance matrix $\hat{\mathbf{D}}$ yields $\frac{2\zeta\sqrt{N}}{\delta(\alpha - \zeta)^2\sqrt{\tau}}$-consistent clustering results, 
where $\zeta = \sqrt{\gamma}\rho + E\sqrt{\gamma\log{\frac{2\tau^2N}{\theta}}}$, $E$ and $\rho$ is defined as before in Lemma~\ref{lemma:sinkhorn_err}, $\gamma$ is the gaussian kernel parameter, and $\tau$ is the sparse parameter.
\end{theorem}
Theorem~\ref{thm:consistency} reveals an important insights: sparsification benefits consistency. Focusing on the sparse parameter $\tau$, the clustering result is $O(\frac{1}{\sqrt{\tau\log\tau}})$-consistent. Therefore, smaller $\tau$ indicates more consistent clustering results. The proof of the theorem is detailed in \ref{app:proof_consistency}

\subsection{Correctness Analysis}\label{sec:correctness}
In the following analysis, we denote $\mu = \{\mu_1,\mu_2,\ldots,\mu_N\}$ as a set of $N$ distributions in $\mathbb{R}^d$ that can be organized into $K$ groups $C_1,C_2,\ldots,C_K$. To establish algorithms' correctness, we first present some definitions.
\begin{definition}[Intra-class neighbor set] Suppose $\mu_i \in C$ is a distribution in cluster $C$, $\tau\text{-NN}_\text{dist}(\mu_i)$ is the set of $\tau$-nearest neighbors of $\mu_i \in \mu$ with respect to some metric $\text{dist}(\cdot,\cdot)$. The inter-class neighbor set of  $\mu_i$ is
\begin{equation}
    \mathcal{N}^{intra}_i := \{\mu_j \in \mu | \mu_j \in C \text{ and } \mu_j \in \tau\text{-NN}_\text{dist}(\mu_i)\}
\end{equation}
\end{definition}
Similarly, we can define \emph{inter-class neighbor set}.
\begin{definition}[Inter-class neighbor set] Suppose $\mu_i \in \mathbb{R}^d$ is a distribution in cluster $C$, $\tau\text{-NN}_\text{dist}(\mu_i)$ is the set of $\tau$-nearest neighbors of $\mu_i \in \mu$ with respect to some metric $\text{dist}(\cdot,\cdot)$. The intra-class neighbor set of  $\mu_i$ is
\begin{equation}
    \mathcal{N}^{inter}_i := \{\mu_j \in \mu | \mu_j \notin C \text{ and } \mu_j \in \tau\text{-NN}_\text{dist}(\mu_i)\}
\end{equation}
\end{definition}

Based on the above definitions, the following definition is used to determine the correctness of clustering.
\begin{definition}[Correctness of Clustering]
\label{clustering_correctness}
Suppose $\mu_i \in \mu$ is a distribution, the clustering is correct with a tolerance of $\xi \geq 0$ if
\begin{enumerate}
    \item $\mathbf{A}_{ij} \geq \max_k(\mathbf{A}_{ik}^{inter}) + \xi$ for any $\mu_j \in \mathcal{N}^{intra}_i$
    \item $\mathbf{A}_{ij} \leq \min_k(\mathbf{A}_{ik}^{intra}) - \xi$ for any $\mu_j \in \mathcal{N}^{inter}_i$
\end{enumerate}
for arbitrary $i = 1, \dots, N$.
\end{definition}

Based on the Definition~\ref{clustering_correctness}, the following theorem gives the guarantee of correct clustering of DDSC$_{\text{Sinkhorn}}$.
\begin{theorem}[Correctness of DDSC$_{\text{Sinkhorn}}$]
\label{thm:correct}
Suppose the clustering is correct with a tolerance of $\xi$ and the Laplacian matrix $\hat{\mathbf{L}}$ has $K$ zero eigenvalues. Then with the probability of at least $1 - \theta$, performing DDSC$_{\text{Sinkhorn}}$ yields correct clustering results if
\begin{equation}
    \sqrt{\gamma}\rho + E\sqrt{\gamma\log{\frac{N\tau^2}{\theta}}} \leq \frac{1}{2}\xi
\end{equation}
where $E$ and $\rho$ is defined the same as Lemma~\ref{lemma:sinkhorn_err}.
\end{theorem}
It is easy to see from the above theorem that smaller $\tau$ is beneficial for correct clustering, which corresponds to the empirical observations. However, it should be noted that if $\tau$ is too small, the information loss will become too large for correct clustering. The proof of Theorem~\ref{thm:correct} can be found in \ref{app:proof_correct}.

\section{Numerical Results}
The proposed methods are evaluated on both synthetic and real datasets. In Section~\ref{exp:syn_data}, we use a synthetic dataset to illustrate the necessity of our DDSC model, demonstrating that other clustering methods fail while only our DDSC performs well. In Section~\ref{exp:real_data}, we perform clustering on three text datasets and two image datasets. The results show that our methods significantly outperform other baselines, further demonstrating the superiority of our approach. Both theoretical time complexity and empirical time cost are presented in Section~\ref{exp:time_cost}.

\subsection{Synthetic Dataset}\label{exp:syn_data}
The synthetic data consists of two types of distributions: one in a square shape and the other in a circular shape, as shown in Figure~\ref{toy_clustering}.

We compare our methods with K-means, spectral clustering (SC), and discrete distribution (D2) clustering \citep{ye2017fast}. The clustering performances of all methods are visualized in Figure~\ref{toy_clustering}, and the adjusted mutual information (AMI) scores \citep{vinh2009information} are reported in Table~\ref{toy-result}. The results show that only our methods can correctly cluster the discrete distributions.

\begin{figure}[ht]
\centering
\centerline{\includegraphics[width=1\linewidth]{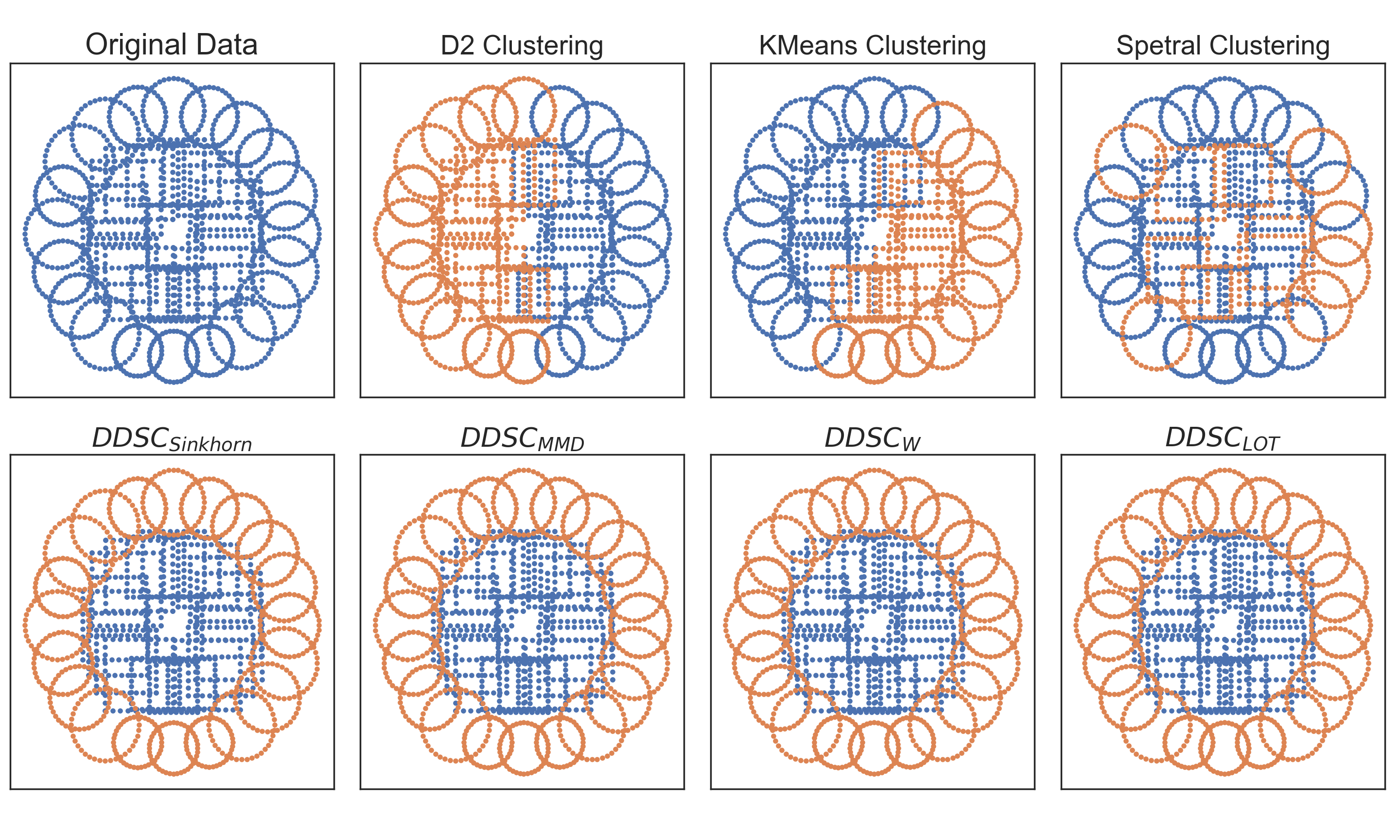}}
\caption{Visualization of synthetic dataset clustering. Different colors indicate different clustering labels produced by the corresponding algorithm. There are 20 square and circular shape distributions, respectively. Each distribution has 40 support points in $\mathbb{R}^2$ space.}
\label{toy_clustering}
\end{figure}


\begin{table}[ht]
\caption{AMI scores of synthetic dataset clustering.}
\label{toy-result}
\centering
\vskip 0.1in
\resizebox{\textwidth}{!}{
\begin{sc}
\begin{tabular}{cccccccc}
\toprule
\multirow{2}{4em}{\textbf{Method}} & \multicolumn{3}{c}{Baselines} & \multicolumn{4}{c}{Proposed} \\
\cmidrule(lr){2-4} \cmidrule(lr){5-8}
 & K-means & SC & D2 & \footnotesize{DDSC$_{\text{LOT}}$} &\footnotesize{DDSC$_{\text{MMD}}$} & \footnotesize{DDSC$_{\text{W}}$} & \footnotesize{DDSC$_{\text{Sinkhorn}}$} \\
 \midrule
 \textbf{AMI} & 0.0 & 0.0 & 0.0 & 1.0 & 1.0 & 1.0 & 1.0 \\
 \bottomrule
\end{tabular} 
\end{sc}
}
\end{table}

\subsection{Real Dataset}\label{exp:real_data}
In our experiments, we use three widely known text datasets to evaluate our methods. Each document is treated as a bag of words, a method extensively explored in previous works.
\begin{itemize}
    \item BBCnews abstract dataset: This dataset is created by concatenating the title and the first sentence of news posts from the BBCNews dataset\footnote[1]{\url{http://mlg.ucd.ie/datasets/bbc.html}}.
    \item BBCSports abstract dataset: Similarly constructed as the BBCNews abstract, this domain-specific dataset also comes from the BBC dataset\footnote[2]{\url{http://mlg.ucd.ie/datasets/bbc.html}}.
    \item Reuters subset: A 5-class subset of the Reuters dataset.
\end{itemize}
We also explore two image datasets to ensure completeness:
\begin{itemize}
    \item MNIST: A subset of 1000 images (100 images for each class) from the well-known MNIST dataset \cite{deng2012mnist}.
    \item Fashion-MNIST: A similar subset from the Fashion-MNIST dataset \cite{xiao2017fashionmnist}.
\end{itemize}

Table~\ref{tab:description_real_data} is a description of datasets.
\begin{table}[ht]
    \centering
    \caption{Description of real text dataset used in our experiments. $N$ is the number of distributions. $n$ is the number of samples in each distribution. $d$ is the dimension of data. $K$ is the number of clusters of the true labels.}
    \label{tab:description_real_data}
    \vskip 0.1in
    \begin{sc}
    \begin{tabular}{lcccc}
    \toprule
    Dataset & $N$ & $d$ & $n$ & $K$ \\
    \midrule
    BBCSports abstr. & 737 & 300 & 16 & 5 \\
    BBCNews abstr. & 2,225 & 300 & 16 & 5 \\
    Reuters & 1,209 & 300 & 16 & 5 \\
    MNIST & 1,000 & 2 & 784 & 10 \\
    Fashion-MNIST & 1,000 & 2 & 784 & 10 \\
    \bottomrule
    \end{tabular}
    \end{sc}
\end{table}

We compare our DDSC$_{\text{MMD}}$, DDSC$_{\text{Sinkhorn}}$, DDSC$_{\text{W}}$, and DDSC$_{\text{LOT}}$ methods with the following baselines: (1) \textit{Spectral Clustering (SC)}, (2) \textit{K-means Clustering (K-means)}, (3) \textit{D2 Clustering (D2)}, and (4) \textit{PD2 Clustering (PD2)}. D2 Clustering is proposed in \citet{ye2017fast}, as described in Section~\ref{sec_WB}. PD2 Clustering \citet{pmlr-v139-huang21f} projects the original support points into a lower dimension first and then applies D2 Clustering. We do not apply PD2 Clustering to image datasets because image data only have two dimensions.

For all methods, we tested the number of clusters $K\in\{5, 6, 7, 8, 9, 10\}$. For DDSC$_{\text{Sinkhorn}}$, we chose the Sinkhorn regularizer $\varepsilon$ between $1$ and $10$. Only the best results are reported.

The numerical experiment results of text datasets and image datasets are shown in Table~\ref{tab:result_text} and Table~\ref{tab:result_img}, respectively.  We also include the experiment results from \citet{pmlr-v139-huang21f}, which use the same text data and preprocessing as our experiments. Our methods significantly outperform barycenter-based clustering (e.g., D2 Clustering) and other baselines for both text and image datasets.

\begin{table}[ht]
    \centering
    \caption{Clustering results (average over five runs) of the real text datasets. * indicates the result is a reference from the work of \citet{pmlr-v139-huang21f}. The bold figures indicate the best 4 results.}
    \vskip 0.1in
    \resizebox{\textwidth}{!}{
    \begin{sc}
    \begin{tabular}{lcccccccc}
        \toprule
        \multirow{2}{3em}{\textbf{Methods}}& \multicolumn{2}{c}{BBC-Sports abstr.} & \multicolumn{2}{c}{BBCNew abstr.} & \multicolumn{2}{c}{Reuters Subsets} & \multicolumn{2}{c}{Average Score} \\
        \cmidrule(lr){2-3} \cmidrule(lr){4-5}\cmidrule(lr){6-7}\cmidrule(lr){8-9}
        & AMI & ARI & AMI & ARI  & AMI & ARI & AMI & ARI\\
        \midrule
        K-means & 0.3408 & 0.3213 & 0.5328 & 0.4950 & 0.4783 & 0.4287 & 0.4506 & 0.4150\\
        K-means$^*$ & 0.4276 & - & 0.3877 & - & 0.4627 & - & 0.4260 & -\\
        SC & 0.3646 & 0.2749 & 0.4891 & 0.4659 & 0.3955 & 0.3265 & 0.4164 & 0.3558\\
        D2 & 0.6234 & 0.4665 & 0.6111 & 0.5572 & 0.4244 & 0.3966 & 0.5530 & 0.4734\\
        D2$^*$ & 0.6510 & - & 0.6095 & - & 0.4200 & - & 0.5602 & -\\
        PD2 & 0.6300 & 0.4680 & 0.6822 & \textbf{0.6736} & 0.4958 & 0.3909 & 0.6027 & 0.5108\\
        PD2$^*$ & \textbf{0.6892} & - & 0.6557 & - & 0.4713 & - & 0.6054 & -\\
        \midrule
        DDSC$_{\text{MMD}}$ & 0.6724 & \textbf{0.5399} & \textbf{0.7108} & 0.6479 & \textbf{0.5803} & \textbf{0.5105} & \textbf{0.6545} & \textbf{0.5661} \\
        DDSC$_{\text{Sinkhorn}}$ & \textbf{0.7855} & \textbf{0.7514} & \textbf{0.7579} & \textbf{0.7642} & \textbf{0.6096} & \textbf{0.5457} & \textbf{0.7177} & \textbf{0.6871} \\
        DDSC$_{\text{W}}$ & \textbf{0.7755} & \textbf{0.7424} & \textbf{0.7549} & \textbf{0.7585} & \textbf{0.6096} & \textbf{0.5457} & \textbf{0.7133} & \textbf{0.6802} \\
        DDSC$_{\text{LOT}}$ & \textbf{0.7150} & \textbf{0.6712} & \textbf{0.7265} & \textbf{0.7499} & \textbf{0.5290} & \textbf{0.4325} & \textbf{0.6580} & \textbf{0.6129} \\
        \bottomrule
    \end{tabular}
    \end{sc}
    }
    \label{tab:result_text}
\end{table}

\begin{table}[ht]
    \centering
    \caption{Clustering results (average over five runs) of the image datasets (subsets).}
    \vskip 0.1in
    \begin{sc}
    \begin{tabular}{lcccccc}
    \toprule
    & \multicolumn{2}{c}{MNIST} & \multicolumn{2}{c}{Fashion-MNIST} \\
    \cmidrule(lr){2-3} \cmidrule(lr){4-5}
    Methods & AMI & ARI & AMI & ARI \\
    \midrule
    K-means & 0.5074 & 0.3779 & 0.5557 & 0.4137 \\
    SC     & 0.4497 & 0.3252 & 0.5135 & 0.3175 \\
    D2     & 0.3649 & 0.2172 & 0.5693 & 0.4174 \\
    \midrule
    DDSC$_{\text{MMD}}$      & \textbf{0.7755} & \textbf{0.6742} & \textbf{0.6553} & \textbf{0.4912} \\
    DDSC$_{\text{Sinkhorn}}$ & \textbf{0.6974} & \textbf{0.6150} & \textbf{0.6332} & 0.4346 \\
    DDSC$_{\text{W}}$        & \textbf{0.7073} & \textbf{0.6199} & 0.6292 & \textbf{0.4691} \\
    DDSC$_{\text{LOT}}$      & 0.6754 & 0.4992 & \textbf{0.6309} & \textbf{0.4469} \\ \bottomrule
    \end{tabular}
    \end{sc}
    \label{tab:result_img}
\end{table}

\subsection{Time Cost}\label{exp:time_cost}
In this section, we present both the time complexity and empirical running time of D2, DDSC$_{\text{MMD}}$, DDSC$_{\text{Sinkhorn}}$, DDSC$_{\text{LOT}}$, and DDSC$_{\text{W}}$. We exclude the running time of PD2 due to its highly time-consuming projection process. Suppose there are $N$ distributions where distribution has $m$ support points, and we try to generate $K$ clusters. The time complexity of different methods in Table~\ref{complexity}.

\begin{table}[ht]
    \caption{Time complexity of different methods. Assume the number of iterations for Sinkhorn divergence, barycenter is $T_1$, $T_2$ respectively.}
    \vskip 0.1in
    \centering
    \begin{sc}
    \begin{tabular}{lc}
        \toprule
        \textbf{Model} & \textbf{Time Complexity} \\
        \midrule
        D2 & $O(T_2Nm^2 + NKm^3\log{m})$ \\
        \midrule
        DDSC$_{\text{MMD}}$ & $O(N^2m^2+N^3)$ \\
        DDSC$_{\text{Sinkhorn}}$ & $O(T_1N^2m^2+N^3)$ \\
        DDSC$_{\text{W}}$ & $O(N^2m^3\log{m} + N^3)$ \\
        DDSC$_{\text{LOT}}$ & $O(Nm^3\log{m} + N^3)$ \\
        \bottomrule
    \end{tabular}
    \end{sc}
    \label{complexity}
\end{table}
We also report the average empirical running time of five experiments in Table~\ref{time}.

\begin{table}[ht]
    \caption{Running time (seconds) of different methods.}
    \centering
    \resizebox{\textwidth}{!}{
    \begin{sc}
    \begin{tabular}{lccccc}
         \toprule
         & BBC-Sports abstr. & BBCNew abstr. & Reuters& MNIST & Fashion-MNIST\\
         \midrule
         D2                       & 99.3 & 640.2  & 305.7 & 910.5 & 2250.0 \\
         \midrule
         DDSC$_{\text{MMD}}$     & \textbf{24.5}  & \textbf{200.8}  & \textbf{80.4}  & \textbf{376.1} & \textbf{364.0}  \\
         DDSC$_{\text{Sinkhorn}}$ & 86.0  & 1252.7 & 755.9 & 845.9 & 3075.4\\
         DDSC$_{\text{W}}$       & 44.9  & 427.7  & 139.9 & 1545.5& 10595.1\\
         DDSC$_{\text{LOT}}$     & \textbf{12.1}  & \textbf{73.8}   & \textbf{38.2}  & \textbf{251.2} & \textbf{251.0}\\
         \bottomrule
    \end{tabular}
    \end{sc}
    }
    \label{time}
\end{table}      

In terms of running time, DDSC$_{\text{MMD}}$ outperforms D2 Clustering because MMD distance has a closed-form solution, whereas solving the Wasserstein barycenter in D2 Clustering is iterative. DDSC$_{\text{W}}$ requires less time than DDSC$_{\text{Sinkhorn}}$ for text datasets, but the opposite is true for image datasets. This is due to the relatively small number of support points ($m$) in text datasets, where the advantage of Sinkhorn divergence over Wasserstein distance is less apparent because it may need many iterations to converge. For data with a large number of support points (large $m$), like images, calculating Wasserstein distances is very time-consuming, making DDSC$_{\text{Sinkhorn}}$ more suitable. DDSC$_{\text{LOT}}$ performs best in terms of running time overall.

\subsection{Incomplete Distance Matrix Clustering}
A natural idea to improve calculation speed is to calculate only part of the distance matrix $\mathbf{D}$ instead of the full matrix. We present the clustering results of DDSC$_{\text{MMD}}$ on the MNIST dataset using an incomplete distance matrix in Table~\ref{tab:random}. Only a certain percentage of distances are calculated, with random selection determining which distances to compute.
\begin{table}[ht]
    \caption{Calculation percentage and clustering performance of DDSC$_{\text{MMD}}$ on MNIST dataset.}
    \vskip 0.1in
    \centering
    \begin{tabular}{lcc}
         \toprule
         & AMI & ARI \\
         \midrule
         90\%  & 0.776    & 0.588 \\
         80\%  & 0.713    & 0.448  \\
         70\%  & 0.778    & 0.587 \\
         60\%  & 0.772    & 0.586\\
         50\%  & 0.770    & 0.581\\
         40\%  & 0.704    & 0.452\\
         30\%  & 0.765    & 0.579 \\
         20\%  & 0.656    & 0.392\\
         10\%  & 0.640    & 0.387 \\
         \bottomrule
    \end{tabular}
    \label{tab:random}
\end{table}  

Comparing Tabel \ref{tab:random} with Table \ref{tab:result_text}, we can see that DDSC$_{\text{MMD}}$ performs better than all baselines even if only 10\% entries of distance matrix are calculated. 

\section{Conclusion}\label{sec:conclusion}
This work proposes a general framework for discrete distribution clustering based on spectral clustering. Using this framework in practice has two key advantages: (1) distributions in real scenarios may not concentrate around centroids but instead may lie in irregular shapes or on multiple manifolds; (2) solving the barycenter problem requires careful initialization and has a high computational cost. Our framework addresses both of these issues. To improve scalability, we introduced a linear optimal transport-based method within this framework. We also provided theoretical analysis to ensure the consistency and correctness of clustering. Numerical results on synthetic datasets, as well as real text and image datasets, demonstrated that our methods significantly outperformed baseline methods in terms of both clustering accuracy and efficiency.

\appendix
\section{Proof for similarity matrix Error bound} \label{proof_similarity_err}
We first propose the following error bound of the similarity matrix. 
\begin{lemma}[Error bound of similarity matrix]
\label{similarity_err}
 With probability at least $1 - N\tau\theta$, the estimated similarity matrix $\mathbf{\hat{A}}$ satisfies
\begin{equation}
    \|\mathbf{A} - \mathbf{\hat{A}}\|_\infty < \sqrt{\gamma}\rho + E\sqrt{\gamma\log{\frac{1}{\theta}}},
\end{equation}
where $\rho$ and $E$ are defined as before in Lemma 4.1 and $\gamma$ was defined in Eq (11).
\end{lemma}

\begin{proof}
To analyze $\|\mathbf{A} - \hat{\mathbf{A}}\|_\infty$, we first look at a single element in $\mathbf{A} - \hat{\mathbf{A}}$,
\begin{equation}
    |a_{ij} - \hat{a}_{ij}| = |\exp(-\gamma D_{ij}^2) - \exp(-\gamma \hat{D}_{ij}^2)|
\end{equation}
Since $|e^{-x^2} - e^{-y^2}| < |x - y|$ for $x, y > 0$,
\begin{equation}
    |a_{ij} - \hat{a}_{ij}| < \gamma^{1/2}|D_{ij} - \hat{D}_{ij}|
\end{equation}
Let $\zeta = \sqrt{\gamma}\rho + E\sqrt{\gamma\log{\frac{1}{\theta}}}$. According to Lemma 4.1, with probability at least $1 - \theta$
\begin{equation}\label{similarity_difference}
    |a_{ij} - \hat{a}_{ij}| < \zeta
\end{equation}
Since we only keep $N\tau$ elements in $\hat{\mathbf{A}}$, we use $N\tau$ random variables $X_i, i = 1, \dots, N\tau$ to denote the elements in $|\mathbf{A} - \hat{\mathbf{A}}|$. By Boole's inequality,
\begin{equation}\label{boole_ineq}
     \mathcal{P}\left(\bigcup_{i=1}^{N\tau}(X_i \geq \zeta)\right) \leq \sum_{i=1}^{N\tau}\mathcal{P}\left(X_i \geq \zeta\right)
\end{equation}
From equation(~\ref{similarity_difference}) we know that $\mathcal{P}(X_i < \zeta) \geq 1-\theta$, so $\mathcal{P}(X_i \geq \zeta) < \theta$. Combine with equation(~\ref{boole_ineq}), we have 
\begin{equation}
    \mathcal{P}\left(\bigcup_{i=1}^{N\tau}(X_i \geq \zeta)\right) < N\tau\theta
\end{equation}
We can derive the error bound as follows,
\begin{equation}
\begin{aligned}
    \mathcal{P}\left(\|\mathbf{A} - \hat{\mathbf{A}}\|_\infty < \zeta\right) &= \mathcal{P}\left(\max\{X_i\}_{i=1}^{N\tau} < \zeta\right) \\
    &= \mathcal{P}\left(X_1 < \zeta, \dots, X_{N\tau} < \zeta\right) \\
    &= 1 - \mathcal{P}\left(\bigcup_{i=1}^{N\tau}(X_i \geq \zeta)\right) \\
    &\geq 1 - N\tau\theta
\end{aligned}
\end{equation}
So the error bound for similarity is
\begin{equation}
    \|\mathbf{A} - \hat{\mathbf{A}}\|_\infty < \zeta
\end{equation}
with probability at least $1 - N\tau\theta$.
\end{proof}

\section{Proof for Laplacian matrix Error Bound}\label{proof_laplacian_err}
Let $\mathbf{L}$ and $\mathbf{\hat{L}}$ be the Laplacian matrices of $\mathbf{A}$ and $\mathbf{\hat{A}}$ respectively. We provide the estimation error bound for $\mathbf{\hat{L}}$ in the following lemma.
\begin{lemma}[Error bound of Laplacian matrix]
\label{laplacian_err}
    Suppose the minimum value of the adjacency matrix $\mathbf{A}$ is $\alpha$. Then with probability at least $1 - 2 \tau^2N\theta$, the following inequality holds $\mathbf{\hat{L}}$ satisfies
    \begin{equation}
        \|\mathbf{L} - \mathbf{\hat{L}}\|_\infty < \frac{2\zeta}{\tau(\alpha-\zeta)^2},
    \end{equation}
    where $\zeta = \sqrt{\gamma}\rho + E\sqrt{\gamma\log{\frac{1}{\theta}}}$.
\end{lemma}

\begin{proof}
By definition, $\mathbf{L} = \mathbf{I} - \mathbf{S}^{-1/2}\mathbf{A}\mathbf{S}^{-1/2}$, $\hat{\mathbf{L}} = \mathbf{I} - \hat{\mathbf{S}}^{-1/2}\hat{\mathbf{A}}\hat{\mathbf{S}}^{-1/2}$. We have
\begin{equation}
    \|\mathbf{L} - \hat{\mathbf{L}}\|_\infty = \|\mathbf{S}^{-1/2}\mathbf{A}\mathbf{S}^{-1/2} - \hat{\mathbf{S}}^{-1/2}\hat{\mathbf{A}}\hat{\mathbf{S}}^{-1/2} \|_\infty
\end{equation}
For $ij$-th entry of matrix $\hat{\mathbf{L}} - \mathbf{L}$,
\begin{equation}
    (\mathbf{L} - \hat{\mathbf{L}})_{ij} = \frac{a_{ij}}{\sqrt{s_i}\sqrt{s_j}} - \frac{\hat{a}_{ij}}{\sqrt{\hat{s}_i}\sqrt{\hat{s}_j}}
\end{equation}
Consider a single element in $\hat{\mathbf{L}} - \mathbf{L}$,
\begin{equation}
\begin{aligned}
    \left| \frac{a_{ij}}{\sqrt{s_i}\sqrt{s_j}} - \frac{\hat{a}_{ij}}{\sqrt{\hat{s}_i}\sqrt{\hat{s}_j}}\right|
    = &\Big|\frac{a_{ij}}{\sqrt{s_i}\sqrt{s_j}} - \frac{a_{ij}}{\sqrt{\hat{s}_i}\sqrt{s_j}} + \frac{a_{ij}}{\sqrt{\hat{s}_i}\sqrt{s_j}} - \frac{\hat{a}_{ij}}{\sqrt{\hat{s}_i}\sqrt{s_j}}\\
    &+ \frac{\hat{a}_{ij}}{\sqrt{\hat{s}_i}\sqrt{s_j}} -\frac{\hat{a}_{ij}}{\sqrt{\hat{s}_i}\sqrt{\hat{s}_j}}\Big| \\
    \leq &\left|\frac{1}{\sqrt{s_i}} - \frac{1}{\sqrt{\hat{s}_i}}\right| \frac{a_{ij}}{\sqrt{s_j}} 
    + \left|a_{ij}-\hat{a}_{ij}\right|\frac{1}{\sqrt{\hat{s}_i}\sqrt{s_j}} 
    \\
    &+ \left|\frac{1}{\sqrt{s_j}} - \frac{1}{\sqrt{\hat{s}_j}}\right|\frac{a_{ij}}{\sqrt{\hat{s}_i}}
\end{aligned}
\end{equation}
Let $f(x) := \frac{1}{\sqrt{x}}$. By Mean Value Theorem,
\begin{equation}
    f(s_i) - f(\hat{s}_i) = \frac{df}{dx}(c)(s_i - \hat{s}_i)
\end{equation}
for some value $c$ between $s_i$ and $\hat{s}_i$. That is,
\begin{equation}
\begin{aligned}
    \left|\frac{1}{\sqrt{s_i}} - \frac{1}{\sqrt{\hat{s}_i}}\right| &\leq \frac{1}{2c^{\frac{3}{2}}}\left|s_i - \hat{s}_i\right|
\end{aligned}
\end{equation}
From Lemma~\ref{similarity_err} we know that $\left|a_{ij} - \hat{a}_{ij}\right| < \zeta$ with probability at least $(1 - \theta)$. Since we only keep $\tau$-largest elements in column of $\mathbf{A}$ and $\hat{\mathbf{A}}$ after sparsing the matrices, by Boole's inequality, $\left|s_{i} - \hat{s}_{i}\right| < \tau\zeta$ with probability at least $(1 - \tau \theta)$. The minimum value of $\mathbf{A}$ is $\alpha$,then by Lemma~\ref{similarity_err}, the minimum possible value of $\hat{\mathbf{A}}$ is $(\alpha - \zeta)$. \textbf{Notice} that in practice we can always assume $\alpha > \zeta$ because the minimum entry of $\hat{\mathbf{A}}$ will never be negative. Therefore, the minimum value of $s_i$ is $\tau\alpha$ and  the minimum value of $\hat{s}_i$ is $\tau(\alpha - \zeta)$. Since $c$ is a value between $s_i$ and $\hat{s}_i$, the minimum value of $c$ is $\tau(\alpha - \zeta)$. Then we have
\begin{equation}\label{eq_part_1}
\begin{aligned}
    \left|\frac{1}{\sqrt{s_i}} - \frac{1}{\sqrt{\hat{s}_i}}\right| \frac{a_{ij}}{\sqrt{s_j}} &< \frac{\tau\zeta}{2(\tau(\alpha-\zeta))^{\frac{3}{2}}}\frac{a_{ij}}{\sqrt{s_j}} \\
    &\leq \frac{\tau\zeta}{2(\tau(\alpha-\zeta))^{\frac{3}{2}}}\frac{1}{\sqrt{s_j}} \\
    &\leq \frac{\tau\zeta}{2(\tau(\alpha-\zeta))^{\frac{3}{2}}}\frac{1}{\sqrt{\tau\alpha}} \\
    &\leq \frac{\tau\zeta}{2(\tau(\alpha-\zeta))^{\frac{3}{2}}}\frac{1}{\sqrt{\tau(\alpha - \zeta)}} \\
    &= \frac{\zeta}{2\tau(\alpha - \zeta)^2}
\end{aligned}
\end{equation}
with probability at least $(1 - \tau \theta)$.
Also, it is easy to see that with probability at least $(1 - \tau \theta)$,
\begin{equation}\label{eq_part_2}
\begin{aligned}
    \left|a_{ij}-\hat{a}_{ij}\right|\frac{1}{\sqrt{\hat{s}_i}\sqrt{s_j}} 
    &< \frac{\zeta}{\sqrt{\hat{s}_i}\sqrt{s_j}} \\
    &\leq \frac{\zeta}{\tau\sqrt{\alpha - \zeta}\sqrt{\alpha}} \\
    &\leq \frac{\zeta}{\tau(\alpha - \zeta)}
\end{aligned}
\end{equation}
With similar steps as inequalities~\ref{eq_part_1}, we can derive that with probability at least $(1 - \tau \theta)$
\begin{equation}\label{eq_part_3}
    \left|\frac{1}{\sqrt{s_j}} - \frac{1}{\sqrt{\hat{s}_j}}\right|\frac{a_{ij}}{\sqrt{\hat{s}_i}}  < \frac{\zeta}{2\tau(\alpha - \zeta)^2}
\end{equation}
Combining equations~\ref{eq_part_1}, \ref{eq_part_2} and \ref{eq_part_3}, we have 
\begin{equation}
\begin{aligned}
    &\left| \frac{a_{ij}}{\sqrt{s_i}\sqrt{s_j}} - \frac{\hat{a}_{ij}}{\sqrt{\hat{s}_i}\sqrt{\hat{s}_j}}\right|\\
    < &\frac{\zeta}{2\tau(\alpha - \zeta)^2} + \frac{\zeta}{\tau(\alpha - \zeta)} + \frac{\zeta}{2\tau(\alpha - \zeta)^2}\\
    =& \frac{\zeta}{\tau}\cdot\frac{1+\alpha-\zeta}{(\alpha-\zeta)^2} \\
    \leq& \frac{2\zeta}{\tau(\alpha-\zeta)^2}
\end{aligned}
\end{equation}
with probability at least $(1 - 2\tau \theta)$. Using Boole's inequality and similar arguments in Appendix~\ref{proof_similarity_err}, we have
\begin{equation}
    \|\mathbf{L} - \hat{\mathbf{L}}\|_\infty \leq \frac{2\zeta}{\tau(\alpha-\zeta)^2}
\end{equation}
with probability at least $(1 - 2\tau^2N\theta)$.
\end{proof}

\section{Proof for Theorem~\ref{thm:consistency}}
\label{app:proof_consistency}
We use the famous Davis-Kahan theorem from matrix perturbation theory to find the bound. The following lemma bounds the difference between eigenspaces of symmetric matrices under perturbations.
\begin{lemma}[Davis-Kahan Theorem]\label{davis-kahan} Let $\mathbf{Z}, \mathbf{H} \in \mathbb{R}^{n \times n}$ be any symmetric matrices. Consider $\hat{\mathbf{Z}} := \mathbf{Z} + \mathbf{H}$ as a pertued version of $\mathbf{Z}$. Let $S_1 \subset \mathbb{R}$ be an interval. Denote by $\lambda_{S_1}(\mathbf{Z})$ the set of eigenvalues of $\mathbf{Z}$ which are contained in $S_1$, and by $\mathbf{V}_1$ the eigenspace corresponding to all those eigenvalues. Denote by $\lambda_{S_1}(\hat{\mathbf{Z}})$ and $\hat{\mathbf{V}}_1$ the analogous quantities for $\hat{\mathbf{Z}}$. Define the distance between $S_1$ and the spectrum of $\mathbf{Z}$ outside of $S_1$ as
$$
\delta = \min\{|\lambda-s|; \lambda \in eig(\mathbf{Z}), \lambda \notin S_1, s \in S_1\}.
$$
Then the distance between the two subspaces $\mathbf{V}_1$ and $\hat{\mathbf{V}}_1$ is bounded by
$$
d(\mathbf{V}_1, \hat{\mathbf{V}}_1) \leq \frac{\|\mathbf{H}\|_F}{\delta}.
$$
\end{lemma}

Following Lemma~\ref{davis-kahan}, in our case, $\delta = |\lambda_{K+1} - \lambda_{K}|$ is the eigen-gap of Laplacian $\mathbf{L}$ and $\mathbf{H} = \hat{\mathbf{L}} - \mathbf{L}$ is the perturbation. To ensure $\xi$-consistency, we only need 
\begin{equation}
    \frac{\|\mathbf{L} - \hat{\mathbf{L}}\|_F}{\delta} \leq \xi.
\end{equation}
According to \ref{proof_laplacian_err}, an arbitrary element in $|\mathbf{L} - \hat{\mathbf{L}}|$ is less than $\frac{2\zeta}{\tau(\alpha-\zeta)^2}$ with probability at least $1 - 2\tau \theta$. So we have
\begin{equation}
\begin{aligned}
    \|\mathbf{L} - \hat{\mathbf{L}}\|_F
    &< \sqrt{N\tau}\cdot\frac{2\zeta}{\tau(\alpha-\zeta)^2} \\
    &= \frac{2\zeta\sqrt{N}}{(\alpha-\zeta)^2\sqrt{\tau}}
\end{aligned}
\end{equation}
with probability at least $1 - 2\tau^2N\theta$.

So a sufficient condition for $\xi$-consistent clustering as defined in Definition 4.2 is
\begin{equation}
    \frac{2\zeta\sqrt{N}}{\delta(\alpha - \zeta)^2\sqrt{\tau}} < \xi.
\end{equation}
Setting new $\theta$ to $2\tau^2N\theta$ completes the proof.

\section{Proof for Theorem 4.7}\label{app:proof_correct}


Consider a distribution $\mu_i$, denote $A_{iu} = \max_k(\mathbf{A}_{ik}^{inter})$ and $A_{iv} = \min_k(\mathbf{A}_{ik}^{intra})$. By assumption, for an arbitrary distributions $\mu_j \in \mathcal{N}^{intra}_i$ and $\mu_{j\prime} \in \mathcal{N}^{inter}_i$,
\begin{equation}\label{d_cond1}
\begin{aligned}
     A_{ij} - A_{iu} &\geq \xi \\
     A_{iv} - A_{ij\prime} &\geq \xi \\
\end{aligned}
\end{equation}
To ensure correct clustering of DDSC$_{\text{Sinkhorn}}$, we need to ensure
\begin{equation}
\begin{aligned}
    \hat{A}_{ij} &\geq \hat{A}_{iu} \\
    \hat{A}_{ij\prime} &\leq \hat{A}_{iv} \\
\end{aligned}
\end{equation}
By proof in Appendix~\ref{proof_similarity_err}, we have
\begin{equation}\label{d_err}
    \left|A_{pq} - \hat{A}_{pq}\right| < \zeta
\end{equation}
with probability at least $1-\theta$ for arbitrary $p, q$. Therefore, to ensure correct clustering, we only need
\begin{equation}\label{d_cond2}
\begin{aligned}
    A_{ij} - \zeta &\geq A_{iu} + \zeta \\
    A_{ij\prime} + \zeta &\leq A_{iv} - \zeta \\
\end{aligned}
\end{equation}
Combining inequality~\ref{d_cond1} and inequality~\ref{d_cond2}, we have
\begin{equation}
    \zeta \leq \frac{1}{2}\xi
\end{equation}
Since we need to satisfy the inequality~\eqref{d_err} for $\mu_j$ in the connected component $\tau\text{-NN}_\text{dist}(\mu_i)$, so the probability is $1 - \tau\theta$. Since there are $N\tau$ distributions to consider in total, the probability is $1 - N\tau^2\theta$.\textbf{}

\section{Robustness Against Noise}
We also explore the robustness of different distance matrices against noise. We construct a distance matrix using 2-Wasserstein distance, linear optimal transport distance, MMD distance and Sinkhorn divergence with regularizers $\varepsilon \in \{5, 10, 50, 100, 500\}$. Gaussian noise $\mathcal{N}(0, \sigma^2 I)$ is added to each datapoint. The relative error under noise level $\sigma$ is defined as
$$
\text{Relative Error} = med\left(\tfrac{D_\sigma - D_0}{D_0}\right)
$$
where $D_\sigma$ is the MMD/Sinkhorn distances of the noisy data, $D_0$ is the ground truth MMD/Sinkhorn distances of the original data and $\text{med}()$ is the median operator. We use the median value instead of the mean value to mitigate the effect of potential numerical errors in the calculated distance matrix due to the added noise. We set the noise $\sigma \in [0, 3.0]$. We set the MMD Gaussian kernel standard deviation to 1. The results are shown in Figure~\ref{noisy_test}.

\begin{figure}[ht]
\centering
\centerline{\includegraphics[width=1\columnwidth]{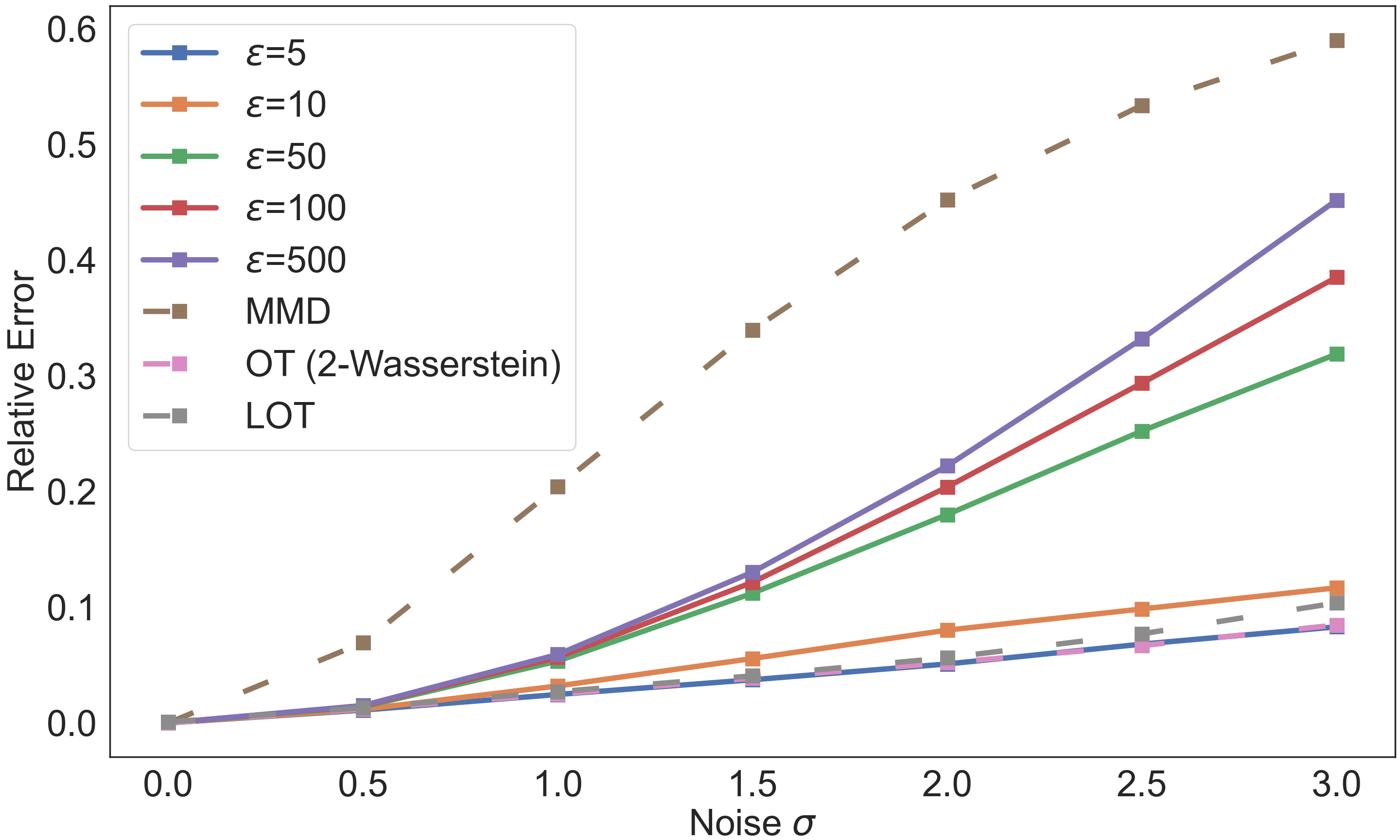}}
\caption{Relative error of distance matrices constructed with different metrics.}
\label{noisy_test}
\end{figure}

From Figure~\ref{noisy_test}, we can see that MMD distance have highest relative error under all noise level. The Sinkhorn divergence with smaller regularizer $\varepsilon$ are more robust against noise. Notice that Sinkhorn divergence interpolate, depending on regularization strength $\varepsilon$, between OT ($\varepsilon = 0$) and MMD ($\varepsilon = +\infty$). So we conjecture that the larger the $\varepsilon$, the less robust the distance matrix against noise.

\section{Preprocessing of real datasets}
We use the pre-trained word-vector dataset GloVe 300d \cite{pennington2014glove} to transform a bag of words into a measure in $\mathbb{R}^{300}$. The TF-IDF scheme normalizes the weight of each word. Before transforming words into vectors, we lower the capital letters, remove all punctuations and stop words and lemmatize each document. We also restrict the number of support points by recursively merging the closest words \citet{pmlr-v139-huang21f}. For image datasets, images are treated as normalized histograms over the pixel locations covered by the digits, where the support vector is the 2D coordinate of a pixel and the weight corresponds to pixel intensity.


\bibliographystyle{elsarticle-num-names} 
\bibliography{refs.bib}

\begin{thebibliography}{44}
\expandafter\ifx\csname natexlab\endcsname\relax\def\natexlab#1{#1}\fi
\providecommand{\url}[1]{\texttt{#1}}
\providecommand{\href}[2]{#2}
\providecommand{\path}[1]{#1}
\providecommand{\DOIprefix}{doi:}
\providecommand{\ArXivprefix}{arXiv:}
\providecommand{\URLprefix}{URL: }
\providecommand{\Pubmedprefix}{pmid:}
\providecommand{\doi}[1]{\href{http://dx.doi.org/#1}{\path{#1}}}
\providecommand{\Pubmed}[1]{\href{pmid:#1}{\path{#1}}}
\providecommand{\bibinfo}[2]{#2}
\ifx\xfnm\relax \def\xfnm[#1]{\unskip,\space#1}\fi
\bibitem[{Shi and Malik(2000)}]{shi2000normalized}
\bibinfo{author}{J.~Shi}, \bibinfo{author}{J.~Malik},
\newblock \bibinfo{title}{Normalized cuts and image segmentation},
\newblock \bibinfo{journal}{IEEE Transactions on pattern analysis and machine
  intelligence} \bibinfo{volume}{22} (\bibinfo{year}{2000})
  \bibinfo{pages}{888--905}.
\bibitem[{Jain et~al.(1999)Jain, Murty, and Flynn}]{Jain_Murty_Flynn_2002}
\bibinfo{author}{A.~K. Jain}, \bibinfo{author}{M.~N. Murty},
  \bibinfo{author}{P.~J. Flynn},
\newblock \bibinfo{title}{Data clustering: A review},
\newblock \bibinfo{journal}{ACM Comput. Surv.} \bibinfo{volume}{31}
  (\bibinfo{year}{1999}) \bibinfo{pages}{264–323}.
  \DOIprefix\doi{10.1145/331499.331504}.
\bibitem[{Ng et~al.(2001)Ng, Jordan, and Weiss}]{ng2002spectral}
\bibinfo{author}{A.~Ng}, \bibinfo{author}{M.~Jordan},
  \bibinfo{author}{Y.~Weiss},
\newblock \bibinfo{title}{On spectral clustering: Analysis and an algorithm},
\newblock in: \bibinfo{editor}{T.~Dietterich}, \bibinfo{editor}{S.~Becker},
  \bibinfo{editor}{Z.~Ghahramani} (Eds.), \bibinfo{booktitle}{Advances in
  Neural Information Processing Systems}, volume~\bibinfo{volume}{14},
  \bibinfo{publisher}{MIT Press}, \bibinfo{year}{2001}.
\bibitem[{Parsons et~al.(2004)Parsons, Haque, and Liu}]{sc2004}
\bibinfo{author}{L.~Parsons}, \bibinfo{author}{E.~Haque},
  \bibinfo{author}{H.~Liu},
\newblock \bibinfo{title}{Subspace clustering for high dimensional data: a
  review},
\newblock \bibinfo{journal}{SIGKDD Explor. Newsl.} \bibinfo{volume}{6}
  (\bibinfo{year}{2004}) \bibinfo{pages}{90--105}.
\bibitem[{Kriegel et~al.(2009)Kriegel, Kr\"{o}ger, and Zimek}]{Kriegel2009}
\bibinfo{author}{H.-P. Kriegel}, \bibinfo{author}{P.~Kr\"{o}ger},
  \bibinfo{author}{A.~Zimek},
\newblock \bibinfo{title}{Clustering high-dimensional data: A survey on
  subspace clustering, pattern-based clustering, and correlation clustering},
\newblock \bibinfo{journal}{ACM Trans. Knowl. Discov. Data} \bibinfo{volume}{3}
  (\bibinfo{year}{2009}). \DOIprefix\doi{10.1145/1497577.1497578}.
\bibitem[{Elhamifar and Vidal(2013)}]{SSC_PAMIN_2013}
\bibinfo{author}{E.~Elhamifar}, \bibinfo{author}{R.~Vidal},
\newblock \bibinfo{title}{Sparse subspace clustering: Algorithm, theory, and
  applications},
\newblock \bibinfo{journal}{IEEE Transactions on Pattern Analysis and Machine
  Intelligence} \bibinfo{volume}{35} (\bibinfo{year}{2013})
  \bibinfo{pages}{2765--2781}.
\bibitem[{Zhang et~al.(2015)Zhang, Zhang, and Liu}]{zhang2015smart}
\bibinfo{author}{X.~Zhang}, \bibinfo{author}{X.~Zhang},
  \bibinfo{author}{H.~Liu},
\newblock \bibinfo{title}{Smart multitask bregman clustering and multitask
  kernel clustering},
\newblock \bibinfo{journal}{ACM Trans. Knowl. Discov. Data}
  \bibinfo{volume}{10} (\bibinfo{year}{2015}). \DOIprefix\doi{10.1145/2747879}.
\bibitem[{Campello et~al.(2015)Campello, Moulavi, Zimek, and
  Sander}]{campello2015hierachical}
\bibinfo{author}{R.~J. G.~B. Campello}, \bibinfo{author}{D.~Moulavi},
  \bibinfo{author}{A.~Zimek}, \bibinfo{author}{J.~Sander},
\newblock \bibinfo{title}{Hierarchical density estimates for data clustering,
  visualization, and outlier detection},
\newblock \bibinfo{journal}{ACM Trans. Knowl. Discov. Data}
  \bibinfo{volume}{10} (\bibinfo{year}{2015}). \DOIprefix\doi{10.1145/2733381}.
\bibitem[{Pandove et~al.(2018)Pandove, Goel, and Rani}]{10.1145/3132088}
\bibinfo{author}{D.~Pandove}, \bibinfo{author}{S.~Goel},
  \bibinfo{author}{R.~Rani},
\newblock \bibinfo{title}{Systematic review of clustering high-dimensional and
  large datasets},
\newblock \bibinfo{journal}{ACM Trans. Knowl. Discov. Data}
  \bibinfo{volume}{12} (\bibinfo{year}{2018}). \DOIprefix\doi{10.1145/3132088}.
\bibitem[{Zhang et~al.(2018)Zhang, Liu, Shen, Shen, and Shao}]{zhang2018binary}
\bibinfo{author}{Z.~Zhang}, \bibinfo{author}{L.~Liu},
  \bibinfo{author}{F.~Shen}, \bibinfo{author}{H.~T. Shen},
  \bibinfo{author}{L.~Shao},
\newblock \bibinfo{title}{Binary multi-view clustering},
\newblock \bibinfo{journal}{IEEE transactions on pattern analysis and machine
  intelligence} \bibinfo{volume}{41} (\bibinfo{year}{2018})
  \bibinfo{pages}{1774--1782}.
\bibitem[{Huang et~al.(2021)Huang, Kang, Xu, and Liu}]{HUANG2021107996}
\bibinfo{author}{S.~Huang}, \bibinfo{author}{Z.~Kang}, \bibinfo{author}{Z.~Xu},
  \bibinfo{author}{Q.~Liu},
\newblock \bibinfo{title}{Robust deep k-means: An effective and simple method
  for data clustering},
\newblock \bibinfo{journal}{Pattern Recognition} \bibinfo{volume}{117}
  (\bibinfo{year}{2021}) \bibinfo{pages}{107996}.
\bibitem[{Kang et~al.(2021)Kang, Lin, Zhu, and Xu}]{kang2021structured}
\bibinfo{author}{Z.~Kang}, \bibinfo{author}{Z.~Lin}, \bibinfo{author}{X.~Zhu},
  \bibinfo{author}{W.~Xu},
\newblock \bibinfo{title}{Structured graph learning for scalable subspace
  clustering: From single view to multiview},
\newblock \bibinfo{journal}{IEEE Transactions on Cybernetics}
  \bibinfo{volume}{52} (\bibinfo{year}{2021}) \bibinfo{pages}{8976--8986}.
\bibitem[{Fan(2021)}]{fankdd21}
\bibinfo{author}{J.~Fan},
\newblock \bibinfo{title}{Large-scale subspace clustering via k-factorization},
\newblock in: \bibinfo{booktitle}{Proceedings of the 27th ACM SIGKDD Conference
  on Knowledge Discovery \& Data Mining}, KDD '21,
  \bibinfo{publisher}{Association for Computing Machinery},
  \bibinfo{address}{New York, NY, USA}, \bibinfo{year}{2021}, p.
  \bibinfo{pages}{342–352}.
\bibitem[{Cai et~al.(2022)Cai, Fan, Guo, Wang, Zhang, and
  Zhang}]{cai2022efficient}
\bibinfo{author}{J.~Cai}, \bibinfo{author}{J.~Fan}, \bibinfo{author}{W.~Guo},
  \bibinfo{author}{S.~Wang}, \bibinfo{author}{Y.~Zhang},
  \bibinfo{author}{Z.~Zhang},
\newblock \bibinfo{title}{Efficient deep embedded subspace clustering},
\newblock in: \bibinfo{booktitle}{Proceedings of the IEEE/CVF Conference on
  Computer Vision and Pattern Recognition}, \bibinfo{year}{2022}, pp.
  \bibinfo{pages}{1--10}.
\bibitem[{Fan et~al.(2022)Fan, Tu, Zhang, Zhao, and Zhang}]{NEURIPS2022_fan}
\bibinfo{author}{J.~Fan}, \bibinfo{author}{Y.~Tu}, \bibinfo{author}{Z.~Zhang},
  \bibinfo{author}{M.~Zhao}, \bibinfo{author}{H.~Zhang},
\newblock \bibinfo{title}{A simple approach to automated spectral clustering},
\newblock in: \bibinfo{editor}{S.~Koyejo}, \bibinfo{editor}{S.~Mohamed},
  \bibinfo{editor}{A.~Agarwal}, \bibinfo{editor}{D.~Belgrave},
  \bibinfo{editor}{K.~Cho}, \bibinfo{editor}{A.~Oh} (Eds.),
  \bibinfo{booktitle}{Advances in Neural Information Processing Systems},
  volume~\bibinfo{volume}{35}, \bibinfo{publisher}{Curran Associates, Inc.},
  \bibinfo{year}{2022}, pp. \bibinfo{pages}{9907--9921}.
\bibitem[{Sun et~al.(2023)Sun, Han, and Fan}]{sun2023laplacian}
\bibinfo{author}{Y.~Sun}, \bibinfo{author}{Y.~Han}, \bibinfo{author}{J.~Fan},
\newblock \bibinfo{title}{Laplacian-based cluster-contractive t-sne for
  high-dimensional data visualization},
\newblock \bibinfo{journal}{ACM Trans. Knowl. Discov. Data}
  \bibinfo{volume}{18} (\bibinfo{year}{2023}).
\bibitem[{Zhao et~al.(2023)Zhao, Yang, and Nie}]{ZHAO2023109836}
\bibinfo{author}{M.~Zhao}, \bibinfo{author}{W.~Yang}, \bibinfo{author}{F.~Nie},
\newblock \bibinfo{title}{Deep multi-view spectral clustering via ensemble},
\newblock \bibinfo{journal}{Pattern Recognition} \bibinfo{volume}{144}
  (\bibinfo{year}{2023}) \bibinfo{pages}{109836}.
\bibitem[{MacQueen et~al.(1967)}]{macqueen1967some}
\bibinfo{author}{J.~MacQueen}, et~al.,
\newblock \bibinfo{title}{Some methods for classification and analysis of
  multivariate observations},
\newblock in: \bibinfo{booktitle}{Proceedings of the fifth Berkeley symposium
  on mathematical statistics and probability}, volume~\bibinfo{volume}{1},
  \bibinfo{organization}{Oakland, CA, USA}, \bibinfo{year}{1967}, pp.
  \bibinfo{pages}{281--297}.
\bibitem[{Bezdek et~al.(1984)Bezdek, Ehrlich, and Full}]{bezdek1984fcm}
\bibinfo{author}{J.~C. Bezdek}, \bibinfo{author}{R.~Ehrlich},
  \bibinfo{author}{W.~Full},
\newblock \bibinfo{title}{Fcm: The fuzzy c-means clustering algorithm},
\newblock \bibinfo{journal}{Computers \& geosciences} \bibinfo{volume}{10}
  (\bibinfo{year}{1984}) \bibinfo{pages}{191--203}.
\bibitem[{Bradley and Mangasarian(2000)}]{bradley2000k}
\bibinfo{author}{P.~S. Bradley}, \bibinfo{author}{O.~L. Mangasarian},
\newblock \bibinfo{title}{K-plane clustering},
\newblock \bibinfo{journal}{Journal of Global optimization}
  \bibinfo{volume}{16} (\bibinfo{year}{2000}) \bibinfo{pages}{23--32}.
\bibitem[{Parsons et~al.(2004)Parsons, Haque, and Liu}]{parsons2004subspace}
\bibinfo{author}{L.~Parsons}, \bibinfo{author}{E.~Haque},
  \bibinfo{author}{H.~Liu},
\newblock \bibinfo{title}{Subspace clustering for high dimensional data: a
  review},
\newblock \bibinfo{journal}{Acm sigkdd explorations newsletter}
  \bibinfo{volume}{6} (\bibinfo{year}{2004}) \bibinfo{pages}{90--105}.
\bibitem[{Xie et~al.(2016)Xie, Girshick, and Farhadi}]{xie2016unsupervised}
\bibinfo{author}{J.~Xie}, \bibinfo{author}{R.~Girshick},
  \bibinfo{author}{A.~Farhadi},
\newblock \bibinfo{title}{Unsupervised deep embedding for clustering analysis},
\newblock in: \bibinfo{editor}{M.~F. Balcan}, \bibinfo{editor}{K.~Q.
  Weinberger} (Eds.), \bibinfo{booktitle}{Proceedings of The 33rd International
  Conference on Machine Learning}, volume~\bibinfo{volume}{48} of
  \textit{\bibinfo{series}{Proceedings of Machine Learning Research}},
  \bibinfo{publisher}{PMLR}, \bibinfo{address}{New York, New York, USA},
  \bibinfo{year}{2016}, pp. \bibinfo{pages}{478--487}.
\bibitem[{Li and Wang(2008)}]{Li_Wang_2007}
\bibinfo{author}{J.~Li}, \bibinfo{author}{J.~Z. Wang},
\newblock \bibinfo{title}{Real-time computerized annotation of pictures},
\newblock \bibinfo{journal}{IEEE Transactions on Pattern Analysis and Machine
  Intelligence} \bibinfo{volume}{30} (\bibinfo{year}{2008})
  \bibinfo{pages}{985--1002}. \DOIprefix\doi{10.1109/TPAMI.2007.70847}.
\bibitem[{Villani(2008)}]{villani2009optimal}
\bibinfo{author}{C.~Villani}, \bibinfo{title}{Optimal transport -- Old and
  new}, volume \bibinfo{volume}{338}, \bibinfo{publisher}{Springer},
  \bibinfo{year}{2008}, pp. \bibinfo{pages}{xxii+973}.
  \DOIprefix\doi{10.1007/978-3-540-71050-9}.
\bibitem[{Agueh and Carlier(2011)}]{wassersteinbarycenter2011}
\bibinfo{author}{M.~Agueh}, \bibinfo{author}{G.~Carlier},
\newblock \bibinfo{title}{Barycenters in the wasserstein space},
\newblock \bibinfo{journal}{SIAM Journal on Mathematical Analysis}
  \bibinfo{volume}{43} (\bibinfo{year}{2011}) \bibinfo{pages}{904--924}.
\bibitem[{Orlin(1988)}]{orlin1988faster}
\bibinfo{author}{J.~Orlin},
\newblock \bibinfo{title}{A faster strongly polynomial minimum cost flow
  algorithm},
\newblock in: \bibinfo{booktitle}{Proceedings of the Twentieth Annual ACM
  Symposium on Theory of Computing}, STOC '88, \bibinfo{publisher}{Association
  for Computing Machinery}, \bibinfo{address}{New York, NY, USA},
  \bibinfo{year}{1988}, p. \bibinfo{pages}{377–387}.
  \DOIprefix\doi{10.1145/62212.62249}.
\bibitem[{Anderes et~al.(2016)Anderes, Borgwardt, and
  Miller}]{Anderes_Borgwardt_Miller_2015}
\bibinfo{author}{E.~Anderes}, \bibinfo{author}{S.~Borgwardt},
  \bibinfo{author}{J.~Miller},
\newblock \bibinfo{title}{Discrete wasserstein barycenters: Optimal transport
  for discrete data},
\newblock \bibinfo{journal}{Mathematical Methods of Operations Research}
  \bibinfo{volume}{84} (\bibinfo{year}{2016}) \bibinfo{pages}{389--409}.
\bibitem[{Kolouri et~al.(2017)Kolouri, Park, Thorpe, Slepcev, and
  Rohde}]{kolouri2017optimal}
\bibinfo{author}{S.~Kolouri}, \bibinfo{author}{S.~R. Park},
  \bibinfo{author}{M.~Thorpe}, \bibinfo{author}{D.~Slepcev},
  \bibinfo{author}{G.~K. Rohde},
\newblock \bibinfo{title}{Optimal mass transport: Signal processing and
  machine-learning applications},
\newblock \bibinfo{journal}{IEEE signal processing magazine}
  \bibinfo{volume}{34} (\bibinfo{year}{2017}) \bibinfo{pages}{43--59}.
\bibitem[{Gretton et~al.(2006)Gretton, Borgwardt, Rasch, Sch\"{o}lkopf, and
  Smola}]{gretton2006kernel}
\bibinfo{author}{A.~Gretton}, \bibinfo{author}{K.~Borgwardt},
  \bibinfo{author}{M.~Rasch}, \bibinfo{author}{B.~Sch\"{o}lkopf},
  \bibinfo{author}{A.~Smola},
\newblock \bibinfo{title}{A kernel method for the two-sample-problem},
\newblock in: \bibinfo{editor}{B.~Sch\"{o}lkopf}, \bibinfo{editor}{J.~Platt},
  \bibinfo{editor}{T.~Hoffman} (Eds.), \bibinfo{booktitle}{Advances in Neural
  Information Processing Systems}, volume~\bibinfo{volume}{19},
  \bibinfo{publisher}{MIT Press}, \bibinfo{year}{2006}.
\bibitem[{Gretton et~al.(2012)Gretton, Borgwardt, Rasch, Sch{\"o}lkopf, and
  Smola}]{gretton2012kernel}
\bibinfo{author}{A.~Gretton}, \bibinfo{author}{K.~M. Borgwardt},
  \bibinfo{author}{M.~J. Rasch}, \bibinfo{author}{B.~Sch{\"o}lkopf},
  \bibinfo{author}{A.~Smola},
\newblock \bibinfo{title}{A kernel two-sample test},
\newblock \bibinfo{journal}{The Journal of Machine Learning Research}
  \bibinfo{volume}{13} (\bibinfo{year}{2012}) \bibinfo{pages}{723--773}.
\bibitem[{Panaretos and Zemel(2019)}]{panaretos2019statistical}
\bibinfo{author}{V.~M. Panaretos}, \bibinfo{author}{Y.~Zemel},
\newblock \bibinfo{title}{Statistical aspects of wasserstein distances},
\newblock \bibinfo{journal}{Annual review of statistics and its application}
  \bibinfo{volume}{6} (\bibinfo{year}{2019}) \bibinfo{pages}{405--431}.
\bibitem[{Cuturi(2013)}]{cuturi2013Sinkhorn}
\bibinfo{author}{M.~Cuturi},
\newblock \bibinfo{title}{Sinkhorn distances: Lightspeed computation of optimal
  transport},
\newblock in: \bibinfo{editor}{C.~Burges}, \bibinfo{editor}{L.~Bottou},
  \bibinfo{editor}{M.~Welling}, \bibinfo{editor}{Z.~Ghahramani},
  \bibinfo{editor}{K.~Weinberger} (Eds.), \bibinfo{booktitle}{Advances in
  Neural Information Processing Systems}, volume~\bibinfo{volume}{26},
  \bibinfo{publisher}{Curran Associates, Inc.}, \bibinfo{year}{2013}.
\bibitem[{Kantorovich(1960)}]{kantorovich1960mathematical}
\bibinfo{author}{L.~V. Kantorovich},
\newblock \bibinfo{title}{Mathematical methods of organizing and planning
  production},
\newblock \bibinfo{journal}{Management science} \bibinfo{volume}{6}
  (\bibinfo{year}{1960}) \bibinfo{pages}{366--422}.
\bibitem[{Genevay et~al.(2016)Genevay, Cuturi, Peyr\'{e}, and
  Bach}]{genevay2016stochastic}
\bibinfo{author}{A.~Genevay}, \bibinfo{author}{M.~Cuturi},
  \bibinfo{author}{G.~Peyr\'{e}}, \bibinfo{author}{F.~Bach},
\newblock \bibinfo{title}{Stochastic optimization for large-scale optimal
  transport},
\newblock in: \bibinfo{editor}{D.~Lee}, \bibinfo{editor}{M.~Sugiyama},
  \bibinfo{editor}{U.~Luxburg}, \bibinfo{editor}{I.~Guyon},
  \bibinfo{editor}{R.~Garnett} (Eds.), \bibinfo{booktitle}{Advances in Neural
  Information Processing Systems}, volume~\bibinfo{volume}{29},
  \bibinfo{publisher}{Curran Associates, Inc.}, \bibinfo{year}{2016}.
\bibitem[{Wang et~al.(2013)Wang, Slep{\v{c}}ev, Basu, Ozolek, and
  Rohde}]{wang2013linear}
\bibinfo{author}{W.~Wang}, \bibinfo{author}{D.~Slep{\v{c}}ev},
  \bibinfo{author}{S.~Basu}, \bibinfo{author}{J.~A. Ozolek},
  \bibinfo{author}{G.~K. Rohde},
\newblock \bibinfo{title}{A linear optimal transportation framework for
  quantifying and visualizing variations in sets of images},
\newblock \bibinfo{journal}{International journal of computer vision}
  \bibinfo{volume}{101} (\bibinfo{year}{2013}) \bibinfo{pages}{254--269}.
\bibitem[{Kolouri et~al.(2016)Kolouri, Tosun, Ozolek, and
  Rohde}]{kolouri2016continuous}
\bibinfo{author}{S.~Kolouri}, \bibinfo{author}{A.~B. Tosun},
  \bibinfo{author}{J.~A. Ozolek}, \bibinfo{author}{G.~K. Rohde},
\newblock \bibinfo{title}{A continuous linear optimal transport approach for
  pattern analysis in image datasets},
\newblock \bibinfo{journal}{Pattern recognition} \bibinfo{volume}{51}
  (\bibinfo{year}{2016}) \bibinfo{pages}{453--462}.
\bibitem[{Kolouri et~al.(2021)Kolouri, Naderializadeh, Rohde, and
  Hoffmann}]{kolouri2020wasserstein}
\bibinfo{author}{S.~Kolouri}, \bibinfo{author}{N.~Naderializadeh},
  \bibinfo{author}{G.~K. Rohde}, \bibinfo{author}{H.~Hoffmann},
  \bibinfo{title}{Wasserstein embedding for graph learning},
  \bibinfo{year}{2021}. \href{http://arxiv.org/abs/2006.09430}{{\tt
  arXiv:2006.09430}}.
\bibitem[{Genevay et~al.(2019)Genevay, Chizat, Bach, Cuturi, and
  Peyr\'{e}}]{genevay2019sample}
\bibinfo{author}{A.~Genevay}, \bibinfo{author}{L.~Chizat},
  \bibinfo{author}{F.~Bach}, \bibinfo{author}{M.~Cuturi},
  \bibinfo{author}{G.~Peyr\'{e}},
\newblock \bibinfo{title}{Sample complexity of sinkhorn divergences},
\newblock in: \bibinfo{editor}{K.~Chaudhuri}, \bibinfo{editor}{M.~Sugiyama}
  (Eds.), \bibinfo{booktitle}{Proceedings of the Twenty-Second International
  Conference on Artificial Intelligence and Statistics},
  volume~\bibinfo{volume}{89} of \textit{\bibinfo{series}{Proceedings of
  Machine Learning Research}}, \bibinfo{publisher}{PMLR}, \bibinfo{year}{2019},
  pp. \bibinfo{pages}{1574--1583}.
\bibitem[{Ye et~al.(2017)Ye, Wu, Wang, and Li}]{ye2017fast}
\bibinfo{author}{J.~Ye}, \bibinfo{author}{P.~Wu}, \bibinfo{author}{J.~Z. Wang},
  \bibinfo{author}{J.~Li},
\newblock \bibinfo{title}{Fast discrete distribution clustering using
  wasserstein barycenter with sparse support},
\newblock \bibinfo{journal}{IEEE Transactions on Signal Processing}
  \bibinfo{volume}{65} (\bibinfo{year}{2017}) \bibinfo{pages}{2317--2332}.
\bibitem[{Vinh et~al.(2009)Vinh, Epps, and Bailey}]{vinh2009information}
\bibinfo{author}{N.~X. Vinh}, \bibinfo{author}{J.~Epps},
  \bibinfo{author}{J.~Bailey},
\newblock \bibinfo{title}{Information theoretic measures for clusterings
  comparison: Is a correction for chance necessary?},
\newblock in: \bibinfo{booktitle}{Proceedings of the 26th Annual International
  Conference on Machine Learning}, ICML '09, \bibinfo{publisher}{Association
  for Computing Machinery}, \bibinfo{address}{New York, NY, USA},
  \bibinfo{year}{2009}, p. \bibinfo{pages}{1073–1080}.
  \DOIprefix\doi{10.1145/1553374.1553511}.
\bibitem[{Deng(2012)}]{deng2012mnist}
\bibinfo{author}{L.~Deng},
\newblock \bibinfo{title}{The mnist database of handwritten digit images for
  machine learning research},
\newblock \bibinfo{journal}{IEEE Signal Processing Magazine}
  \bibinfo{volume}{29} (\bibinfo{year}{2012}) \bibinfo{pages}{141--142}.
\bibitem[{Xiao et~al.(2017)Xiao, Rasul, and Vollgraf}]{xiao2017fashionmnist}
\bibinfo{author}{H.~Xiao}, \bibinfo{author}{K.~Rasul},
  \bibinfo{author}{R.~Vollgraf}, \bibinfo{title}{Fashion-mnist: a novel image
  dataset for benchmarking machine learning algorithms}, \bibinfo{year}{2017}.
  \href{http://arxiv.org/abs/1708.07747}{{\tt arXiv:1708.07747}}.
\bibitem[{Huang et~al.(2021)Huang, Ma, and Lai}]{pmlr-v139-huang21f}
\bibinfo{author}{M.~Huang}, \bibinfo{author}{S.~Ma}, \bibinfo{author}{L.~Lai},
\newblock \bibinfo{title}{Projection robust wasserstein barycenters},
\newblock in: \bibinfo{editor}{M.~Meila}, \bibinfo{editor}{T.~Zhang} (Eds.),
  \bibinfo{booktitle}{Proceedings of the 38th International Conference on
  Machine Learning}, volume \bibinfo{volume}{139} of
  \textit{\bibinfo{series}{Proceedings of Machine Learning Research}},
  \bibinfo{publisher}{PMLR}, \bibinfo{year}{2021}, pp.
  \bibinfo{pages}{4456--4465}.
\bibitem[{Pennington et~al.(2014)Pennington, Socher, and
  Manning}]{pennington2014glove}
\bibinfo{author}{J.~Pennington}, \bibinfo{author}{R.~Socher},
  \bibinfo{author}{C.~Manning},
\newblock \bibinfo{title}{{G}lo{V}e: Global vectors for word representation},
\newblock in: \bibinfo{booktitle}{Proceedings of the 2014 Conference on
  Empirical Methods in Natural Language Processing ({EMNLP})},
  \bibinfo{publisher}{Association for Computational Linguistics},
  \bibinfo{address}{Doha, Qatar}, \bibinfo{year}{2014}, pp.
  \bibinfo{pages}{1532--1543}. \DOIprefix\doi{10.3115/v1/D14-1162}.

\end{thebibliography}





\end{document}